\documentclass[11pt,reqno,letter]{amsart}

\usepackage{amssymb}
\usepackage[dvipdfmx]{graphicx}
\usepackage{hyperref}
\usepackage[usenames]{color}
\usepackage{caption}
\usepackage{subcaption}
\usepackage{pifont}
\usepackage{bm}
\usepackage{url}
\usepackage{setspace}
\usepackage[margin=1in]{geometry}

\usepackage[linesnumbered,ruled,vlined]{algorithm2e}

\usepackage{algorithmicx}

\usepackage[normalem]{ulem}

\usepackage{float}

\usepackage{xfrac}

\theoremstyle{plain}
\newtheorem{theorem}{Theorem}[section]

\newtheorem{proposition}[theorem]{Proposition}

\theoremstyle{definition}

\newtheorem{assumption}[theorem]{Assumption}

\numberwithin{equation}{section}
\numberwithin{theorem}{section}
\numberwithin{table}{section}
\numberwithin{figure}{section}

\DeclareMathOperator*{\argmax}{arg\,max}
\DeclareMathOperator*{\argmin}{arg\,min}

\def\@fnsymbol#1{\ensuremath{\ifcase#1\or *\or \dagger\or \ddagger\or
   \mathsection\or \mathparagraph\or \|\or **\or \dagger\dagger
   \or \ddagger\ddagger \else\@ctrerr\fi}}
\title{Calibration of distributionally robust empirical optimization models}

\author[Gotoh]{Jun-ya Gotoh\textsuperscript{$\dagger$}}

\author[Kim]{Michael Jong Kim\textsuperscript{$\ddagger$}}

\author[Lim]{Andrew E.B. Lim\textsuperscript{$*$}}

\dedicatory{\textsuperscript{$\dagger$}Department of Industrial and Systems Engineering, Chuo University, Tokyo, Japan. Email: jgoto@indsys.chuo-u.ac.jp \\ \textsuperscript{$\ddagger$}Sauder School of Business, University of British Columbia, Vancouver, Canada. Email: mike.kim@sauder.ubc.ca \\ \textsuperscript{$*$}Department of Analytics and Operations, Department of Finance, Institute for Operations Research and Analytics,  National University of Singapore, Singapore. Email: andrewlim@nus.edu.sg
}
\date{February 14, 2020}
\begin{document}
%%%%%%%%%%%%%%%%

%%%%%%%%%%%%%%%%
\begin{abstract}
We study the out-of-sample properties of robust empirical optimization problems with smooth $\phi$-divergence penalties and smooth concave objective functions, and develop a theory for data-driven calibration of the non-negative ``robustness parameter" $\delta$ that controls the size of the deviations from the nominal model. Building on the intuition that robust optimization reduces the sensitivity of the expected reward to errors in the model by controlling the spread of the reward distribution, we show that the first-order benefit of ``little bit of robustness" (i.e., $\delta$ small, positive) is a significant reduction in the variance of the out-of-sample reward while the corresponding impact on the mean is almost an order of magnitude smaller. One implication is that substantial variance (sensitivity) reduction  is possible at little cost if the robustness parameter is properly calibrated. To this end, we introduce the notion of a robust mean-variance frontier to select the robustness parameter and show that it can be approximated using resampling methods like the bootstrap. Our examples show that robust solutions resulting from ``open loop" calibration methods (e.g., selecting a $90\%$ confidence level regardless of the data and objective function) can be very conservative out-of-sample, while those corresponding to the robustness parameter that optimizes an estimate of the out-of-sample expected reward (e.g., via the bootstrap) with no regard for the variance are often insufficiently robust.
\end{abstract}
%%%%%%%%%%%%%%%%

%%%%%%%%%%
\maketitle
%%%%%%%%%%

%%%%%%%%%%%%%%%%

\section{Introduction}
\label{sec:introduction}

Empirical optimization (or sample average approximation (SAA)) is an approach for expected value optimization using the empirical distribution as an alternative model to the unknown true distribution. Misspecification can occur because of incorrect modeling assumptions or estimation uncertainty, and decisions that are made on the basis of an incorrect model can perform poorly out-of-sample if misspecification is ignored. Distributionally robust optimization (DRO) accounts for misspecification  in the in-sample problem by optimizing against worst-case perturbations from the ``nominal model."

DRO models are typically parameterized by an ``ambiguity parameter" $\delta$ that controls the size of the deviations from the nominal model in the worst-case problem. The ambiguity parameter  may appear as the confidence level of an uncertainty set or a penalty parameter that multiplies some measure of deviation between alternative probability distributions and the nominal in the worst-case objective. The parameter $\delta$ defines a family of worst-case solutions $\{x_n(\delta) : \delta\geq 0\}$, where $\delta=0$ is empirical/sample-average optimization (i.e., no robustness) with solutions becoming increasingly conservative as $\delta$ increases. (Here, $n$ denotes the size of the historical data set used to construct the robust optimization problem.) The choice of $\delta$ clearly determines the out-of-sample performance of the robust solution. The goal of this paper is to  study the out-of-sample properties of DRO and  to develop a theory for data-driven calibration of the ``robustness parameter" for  worst-case maximization problems with concave reward functions.

\subsection*{Summary of contributions}

\begin{enumerate}
\item {\bf DRO is a multi-objective problem.}
In the context of our model, we show that the sensitivity of the expected reward to worst-case deviations from the nominal is equal to the variance of the reward. Building on this insight and the results from \cite{GKL}, we show that DRO is a multi-objective problem where the goal is to maximize the {\it expected reward} while controlling the {\it worst-case sensitivity} (variance) to misspecification. The trade-off between mean and sensitivity (variance) is controlled by the robustness parameter $\delta$, which suggests that both objectives should be considered when selecting $\delta$.

\item {\bf Out-of-sample properties of DRO and worst-case sensitivity reduction.}
We show that solutions of SAA and DRO are asymptotically normal ($n$ large), that a little robustness ($\delta$ small) always reduces the sensitivity (variance) of the out-of-sample reward, and that  sensitivity (variance) reduction is almost an order of magnitude larger than the impact on the mean when $\delta$ is small. We also show that out-of-sample expected reward under the worst-case solution can sometimes be larger than that of SAA, though this improvement is typically small relative to sensitivity/variance reduction (see \cite{BLSZZ2013,KL,LorcaSun} for empirical observations of this phenomenon and \cite{AP} for the analysis in a one-dimensional quadratic setting).

\item {\bf Calibration of the robustness parameter.}  The problem of calibrating $\delta$ in DRO is analogous to the selection of the free parameters in various Machine Learning algorithms (e.g., the penalty parameter in LASSO, Ridge regression, SVM, etc.), which are commonly tuned by optimizing an estimate of out-of-sample loss via cross-validation, bootstrap or some other resampling approach. While $\delta$ could be selected by optimizing an estimate of the out-of-sample expected reward, this ignores the multi-objective character of DRO which suggests that both the mean and variance should be taken into account. We show how the out-of-sample robust mean-variance frontier can be constructed using resampling techniques. Consistent with our theory, our examples show that substantial variance reduction is possible with minimal impact on the expected reward if $\delta$ is chosen appropriately. The values of $\delta$ obtained from our approach are typically much smaller than those implied by standard confidence levels (e.g., $90\%, 95\%$) advocated in the literature, and larger than those obtained by maximizing only the mean.

\end{enumerate}

\subsection*{Other related literature}

Several papers \cite{DGN16,GKL,Lam,DN} discuss DRO from the perspective of empirical optimization with the variance of the reward as a regularizer. Most related is \cite{GKL} which studies the connection between DRO with (sufficiently smooth) $\phi$-divergence penalties and mean-variance problems when the robustness parameter is small. The most important difference is that \cite{GKL} studies the in-sample DRO problem while this paper is an out-of-sample analysis. Specifically, we characterize the impact of data variability on the DRO solution and the mean and variance of the out-of-sample reward, which forms the basis of resampling-based calibration  methods for selecting the robustness parameter $\delta$. Our analysis also reveals the mechanism by which the out-of-sample expected reward under DRO solutions can exceed that of SAA, a phenomena that has been observed empirically in a number of papers but not fully explained (we discuss this in more detail below).

A common theme in \cite{DGN16,GKL,Lam,DN} and the present paper is the relationship between DRO and variance regularization of SAA, and it is natural to ask {\it why} such a relationship even exists.  In this regard, we show that the in-sample variance of the reward is equal to the sensitivity of the mean under worst-case deviations from the nominal, so optimizing against the worst-case in DRO is almost the same as maximizing the expected reward while controlling worst-case sensitivity to model misspecification (variance). This formalizes discussion in \cite{GKL} and provides insight into why DRO and variance regularization of SAA are related. Variance regularization is also proposed in \cite{ELV} as a way to reduce the impact of data variability of the out-of-sample performance of the solution of mean-CVaR portfolio choice problems, though connections to worst-case optimization are not discussed; see also \cite{GKL} for the relationship between worst-case CVaR and variance regularization.

Also related is \cite{DN} which develops finite sample probabilistic guarantees for the out-of-sample expected reward generated by robust solutions.
One important difference is that \cite{DN} studies the out-of-sample  expected reward while the focus of our paper is the variance/sensitivity reduction properties of DRO and the implications for calibration. Additionally, while the probabilistic guarantees in \cite{DN} formalize the relationship between data size, model complexity, robustness, and out-of-sample expected reward, these results depend on quantities that are difficult to compute (e.g., the covering number or the VC Dimension), require a bounded objective function, and come with the usual concerns that performance bounds of this nature are loose \cite{AML}. Note too that \cite{DN} is a finite sample analysis whereas ours is asymptotic (large sample size), though calibration experiments for small data sets produce results that are consistent with our large sample theory. We also mention
\cite{DGN16}, which provides confidence intervals for the optimal objective value and shows consistency of solutions using Owen's empirical likelihood theory, and \cite{Lam} which studies the sensitivity of estimates of the expected value of random variables to worst-case perturbations of a simulation model.

Connections between robustness and regularization are studied in \cite{BC,BM,GCK,XCM} where various worst-case regression and classification problems are shown to be equivalent to the regularized regression and classification problems (Ridge, Lasso, etc.). Compared to this paper and also \cite{DN,DGN16,GKL,Lam}, one important difference is that the regularizers in \cite{BC,BM,GCK,XCM} act on the solution while the sensitivity/variance penalty in this paper acts on the reward. The equivalence with solution regularization in \cite{BC,BM,GCK,XCM} implies that worst-case regression and classification reduces the variance of the {\it solution} (regression estimate) by shrinking it towards the origin. It turns out, however, that this is not a general property of DRO, but is  driven by the objective function associated with regression and classification. Indeed, we  show in Section \ref{sec:stats_robust_sol} that there are well behaved concave reward functions where the impact of DRO on the SAA solution is opposite to that of a solution regularizer, biasing it away from the origin and increasing its variance.

We comment briefly on various approaches to calibrating DRO models in the literature. One approach is to use {\it high confidence uncertainty sets} that include the true data generating model with high probability, where the confidence level (typically $90\%$, $95\%$, or $99\%$) is a primitive of the model \cite{ben2013,bgk-SAA,DY-2010,DGN16,Lam}.  This has been refined by \cite{bgk13} which develops methods for finding the smallest such uncertainty set, though the confidence level remains a primitive of the model. One concern with this approach is that it is ``open loop", with confidence levels chosen independent of the data and the objective function. There is, however, no reason why these particular confidence levels should have anything to do with good out-of-sample performance.

A second approach selects the robustness parameter $\delta$ by {\it optimizing an estimate of out-of-sample expected reward} using a resampling procedure like bootstrap or cross-validation. While $\delta$ now depends on the data and the objective function (it is no longer ``open loop"), it ignores the multi-objective character of DRO. Indeed, our results show that sensitivity/variance reduction is the first-order benefit in DRO, so optimizing just the mean corresponds to selecting $\delta$ on the basis of (almost) a second-order effect. While optimizing the mean can produce a solution that has a higher out-of-sample expected reward than empirical optimization in certain applications, there are other applications where, by the same criteria, it is optimal not to be robust (i.e., $\delta=0$). In either case, there are typically larger values of $\delta$ where further reduction in the sensitivity/variance is possible with negligible impact on the mean. In contrast  to the ``high-confidence" approach where classical confidence levels of  $90\%$, $95\%$, or $99\%$ produce overly conservative solutions, this approach is likely to be  insufficiently robust.

Finally, in the {\it satisfycing approach} of \cite{BDS}, the decision maker specifies a target level $T$ and finds the ``most robust" decision that achieves this target under the worst case {\it in-sample} alternative model. That is, he/she chooses the largest robustness parameter $\delta$ under which the worst-case expected reward exceeds the target $T$, so the target $T$  is a primitive of the problem and the confidence level $\delta$ is optimized.

One surprising property of DRO is that the out-of-sample expected reward can sometimes exceed that of the SAA solution, even though it is a worst-case approach. This has been observed empirically \cite{BLSZZ2013,KL,LorcaSun}, and is studied in \cite{AP} for a quadratic reward and one-dimensional decision variable. We show that the difference between the out-of-sample expected reward under its ``true" maximizer and the SAA decision is nothing but the ``gap" in Jensen's inequality due to the variability of the solution (cf. Figure \ref{fig:Jensen1}), and that the improvement in the out-of-sample expected reward under DRO, should it occur, is precisely the reduction in this ``gap" under the DRO solution. Our results show that the change in the expected reward is likely to be small relative to the reduction in worst-case sensitivity, which is the primary benefit of the DRO model.

Finally, we mention the paper \cite{WZZ} which studies the asymptotic properties of stochastic optimization problems with risk-sensitive objectives.

\subsection*{Notation}

Given any function $f$ defined over the positive reals, we write $f(\delta) = o(\delta)$ (as $\delta \rightarrow 0)$ provided $\lim_{\delta\rightarrow0}f(\delta)/\delta = 0$, and we write $f(\delta) = O(\delta)$ (as $\delta \rightarrow 0)$ provided there exists $M, \delta_{0} > 0$ such that $f(\delta) \leq M\delta$ for all $\delta \leq \delta_{0}$. Given two stochastic processes $\{X_{n}\}$ and $\{Y_{n}\}$ defined on a common probability space, we write $X_{n} = o_{P}(Y_{n})$ provided for any $\epsilon > 0$, $\lim_{n\rightarrow\infty}\mathbb{P}(|X_{n}/Y_{n}|>\epsilon) = 0$.

\subsection*{Outline of the paper}

We introduce the in-sample DRO model and the notion of worst-case sensitivity in Section \ref{sec:REO}. We also show that worst-case sensitivity is equal to the variance of the reward for all sufficiently smooth $\phi$-divergence and that our DRO problem is equivalent to a multi-objective mean-sensitivity problem when the robustness parameter is small. We characterize the statistical properties of the in-sample solutions in Section \ref{sec:stats_robust_sol}. We also show that it is possible for DRO to increase the variance of the solution relative to empirical optimization. We characterize the mean and variance of the associated out-of-sample rewards for empirical optimization ($\delta = 0$) in Section \ref{sec:out-of-sample-emp} and robust optimization ($\delta > 0$) in Section \ref{sec:out-of-sample-robust}. In the case of DRO, we show that out-of-sample variance reduction is almost an order of magnitude larger than any change in the mean when the robustness parameter is small.  We introduce our calibration method, which is based on resampled estimates of the mean and variance, in Section \ref{sec:calibration}, and apply this to three examples in Section \ref{sec:applications}. We conclude in Section \ref{sec:conclusions}.

%%%%%%%%%%%%%%%%

\section{Robust empirical optimization}
\label{sec:REO}

Let $f(x,\,Y)$ be a real-valued reward function, where $x\in{\mathbb R}^{d}$ is the decision variable and $Y\in\mathbb{R}^{l}$ a random vector with distribution ${\mathbb P}$. If the distribution $\mathbb P$ of $Y$ is known, we can solve
\begin{equation}
\max_x {\mathbb E}_{\mathbb P}\big[f(x,\,Y)\big],
\label{eq1}
\end{equation}
(assumptions on $f$ will be given later). In many situations, the distribution $\mathbb P$ is not known. Instead, we have historical data $Y_1,\cdots,\,Y_n$ generated {\it i.i.d.} from $\mathbb P$, and would like to use this data set to make a decision that performs well out-of-sample.

When the distribution $\mathbb P$ of $Y$ is not known, it is natural to optimize a sample average approximation of the objective
\begin{equation}
x_n(0) := \arg \max_x \Big\{ \mathbb{E}_{{\mathbb P}_n} \big[f(x,\,Y)\big] \equiv \frac{1}{n}\sum_{i=1}^n f(x,\,Y_i) \Big\},
\label{emp_n}
\end{equation}
where ${\mathbb P}_n$ is the empirical distribution associated with $Y_1,\cdots,\,Y_n$. In many situations, however,  $x_n(0)$ performs poorly out-of-sample due to differences between the empirical estimate ${\mathbb P}_n$ and the actual population distribution $\mathbb P$.

One approach to this problem is to optimize worst-case versions of the empirical optimization problem \eqref{emp_n}. Specifically, let $\delta>0$ be a positive constant and
\begin{equation}
x_n(\delta) := \arg \max_x \, \min_{{\mathbb Q}} \Big\{\mathbb{E}_{{\mathbb Q}}\big[f(x,Y)\big] + \frac{1}{\delta} \mathcal{H}_{\phi}({\mathbb Q} \,|\, {\mathbb{P}}_n)\Big\}
\label{eq:robust_empirical}
\end{equation}
be the solution of the DRO problem where
\begin{equation}
\mathcal{H}_\phi(\mathbb{Q} \,|\, {\mathbb{P}}_n):=
\left\{\begin{array}{ll}
\sum\limits_{i:{p}^{n}_i>0} {p}^n_i\phi\left(\frac{q_i}{{p}^{n}_i}\right),&\sum\limits_{i:{p}^{n}_i>0}q_i=1, q_i\geq 0,\\
+\infty,&\mbox{otherwise},\\
\end{array}\right.
\label{phi-div}
\end{equation}
is the $\phi$-divergence of ${\mathbb Q}=[q_1,\cdots,\,q_n]$ relative to ${\mathbb P}_n = [p_1^{n},\cdots,\,p_n^{n}]$. While there are important examples of $\phi$-divergence that are not smooth, we assume throughout this paper that $\phi$ is a convex function such that
\begin{equation}
\mathrm{dom}\,\phi\subset[0,\infty),\;
\phi(1)=0,\; \phi'(1)=0,\;  \mbox{and}\; \phi{''}(1)>0
\label{eq:phi-cond}
\end{equation}
(additional assumptions will be stated as required). The $\phi$-divergence $\mathcal{H}_\phi(\mathbb{Q} \,|\, {\mathbb{P}}_n)$ measures the deviation of the alternative distribution $\mathbb Q$ from the nominal ${\mathbb P}_n$ and the worst-case model \eqref{eq:robust_empirical} accounts for errors in the empirical distribution ${\mathbb P}_n$ by optimizing against worst-case perturbations. The size of the perturbations in \eqref{eq:robust_empirical} is controlled by the robustness parameter $\delta$ which determines the penalty on the adversary for deviating from the nominal. Note that $\delta=0$ gives the empirical model \eqref{emp_n} and  solutions become ``more conservative"  as $\delta$ increases.

It is shown in \cite{GKL} that if $\phi(z)$ is sufficiently smooth, the worst-case optimization problem \eqref{eq:robust_empirical} is almost the same as an empirical mean-variance problem. Specifically, if
$\phi(z)$ satisfies \eqref{eq:phi-cond} and is also twice continuously differentiable in the neighborhood of $z=1$, then
\begin{equation}
\min_{{\mathbb Q}}
\Big\{
\mathbb{E}_{{\mathbb Q}}\big[f(x,Y)\big] + \frac{1}{\delta} \mathcal{H}_{\phi}({\mathbb Q} \,|\, \mathbb{P}_n)
\Big\}
= \mathbb{E}_{{\mathbb P}_n}\big[f(x,Y)\big] - \frac{\delta}{2\phi''(1)}\mathbb{V}_{{\mathbb P}_n}\big[f(x,\,Y)\big]  + o(\delta),
\label{eq:robust-mv-general}
\end{equation}
where
\begin{eqnarray*}
\mathbb{V}_{{\mathbb P}_n}\big[f(x,\,Y)\big] := \frac{1}{n}\sum_{i=1}^n\Big(f(x,\,Y_i)-\mathbb E_{{\mathbb P}_n}\big[f(x,\,Y_i)\big]\Big)^2
\end{eqnarray*}
is the variance of the reward under the empirical distribution. This shows that worst-case optimization \eqref{eq:robust_empirical} is actually a multi-objective problem that trades off between maximizing the expected reward ${\mathbb E}_{{\mathbb P}_n}[f(x,\,Y)]$ and minimizing its variance $\mathbb{V}_{{\mathbb P}_n}\big[f(x,\,Y)\big]$.

While the mean-variance objective has been widely adopted, it is generally for reasons unrelated to model uncertainty, and it is natural to ask {\it why} it now appears in relation to a worst-case problem.  It turns out that the variance of the reward measures worst-case sensitivity of the nominal model to misspecification. Intuitively, a decision $x$ is sensitive to misspecification if small deviations from the nominal ${\mathbb P}_n$ can have a big impact on the expected reward ${\mathbb E}_{{\mathbb P}_n}[f(x,\,Y)]$. Sensitivity is large if the reward distribution has a large spread since perturbations that affect the tail can have a big impact on the mean.
The worst-case optimizer \eqref{eq:robust-mv-general} tries to optimize the expected reward while controlling the worst-case sensitivity to misspecification (spread).

This intuition can be formalized. Given a nominal model ${\mathbb P}_n$ and reward $f(x,\,Y)$, consider the family of distributions  $\{{\mathbb Q}(\delta):\,\delta\geq 0\}$ where
\begin{align}
{\mathbb Q}(\delta) &
:= \left\{
\begin{array}{cl}
\begin{displaystyle}\arg\min_{\mathbb Q}\Big\{  \sum_{i=1}^nq_i f(x,\,Y) +  \frac{1}{\delta} \sum_{i=1}^n {p}^n_i \phi\Big(\frac{q_i}{{p}^n_i}\Big)\Big\}, \end{displaystyle}& \delta>0,\\ [10pt]
{\mathbb P}_n, & \delta=0.
\end{array}\right.
\label{eq:worst-case-Q}
\end{align}
Intuitively, $\{{\mathbb Q}(\delta):\,\delta\geq 0\}$ is a family of worst-case deviations from the nominal, with the size of the deviation increasing in $\delta$ (corresponding to a smaller penalty on the $\phi$-divergence in \eqref{eq:worst-case-Q}). Although we have suppressed this from the notation, $\{{\mathbb Q}(\delta):\,\delta\geq 0\}$ depends on the nominal ${\mathbb P}_n$ as well as the reward $f(x,\,Y)$.
Given \eqref{eq:worst-case-Q}, we define the worst-case sensitivity of the expected reward ${\mathbb E}_{{\mathbb P}_n}[f(x,\,Y)]$ relative to nominal model ${\mathbb P}_n$ as
\begin{align}
{\mathcal S}_{{\mathbb P}_n}(f(x,\,Y)) &
:= - \frac{\mathrm{d}}{\mathrm{d}\delta}{\mathbb E}_{{\mathbb Q}(\delta)}[f(x,\,Y)]\Big|_{\delta=0} \nonumber  \\[5pt]
& = - \lim_{\delta\downarrow 0}\frac{{\mathbb E}_{{\mathbb Q}(\delta)}[f(x,\,Y)] -{\mathbb E}_{{\mathbb P}_n}[f(x,\,Y)]}{\delta}.
\label{eq:sensitivity}
\end{align}
The following result shows that the worst-case sensitivity is equal to the variance. We require the following assumption.
\begin{assumption}
\label{ass2}
$\phi:{\mathbb R}\rightarrow{\mathbb R}\cup\{+\infty\}$ is a closed convex function such that $\phi(z)\geq\phi(1)=0$ for $z\geq 0$, and $\phi(z)=+\infty$ for $z<0$, and twice continuously differentiable around $z=1$ with $\phi''(1)>0$.
\end{assumption}
The proof of the following result can be found in Appendix \ref{sec:sensitivity}.
\begin{theorem} \label{thm:sensitivity}
Suppose $\phi(z)$ satisfies Assumption \ref{ass2}. Let $\{{\mathbb Q}(\delta):\,\delta \geq 0\}$ be the family of worst case measures defined by \eqref{eq:worst-case-Q}. Then the worst-case sensitivity of the expected reward ${\mathbb E}_{{\mathbb P}_n}[f(x,\,Y)]$ relative to the nominal ${\mathbb P}_n$ \eqref{eq:sensitivity} satisfies
\begin{equation*}
{\mathcal S}_{{\mathbb P}_n}(f(x,\,Y)) = \frac{1}{\phi''(1)}{\mathbb V}_{{\mathbb P}_n}[f(x,\,Y)].
\end{equation*}
\end{theorem}

It can be also shown that the sensitivity of the expected reward under different models of misspecification (e.g., Wasserstein and CVaR) corresponds to different measures of spread, and that the associated worst-case problems are mean-sensitivity problems. These results will be reported elsewhere. Worst-case sensitivity is also discussed in \cite{Lam} for simulated estimates on an expectation when  the family of alternative models is defined through a constraint on relative entropy. In this regard, Theorem \ref{thm:sensitivity} holds for all sufficiently smooth $\phi$-divergence and not just for relative entropy. Variance penalties are proposed in \cite{ELV} as a way to control the sensitivity of solutions of an empirical mean-CVaR portfolio choice problem to data variability, but the connection to worst-case optimization is not discussed.

\subsection*{Calibration}

Robust optimization \eqref{eq:robust_empirical} defines a family of policies $\{x_n(\delta):\,\delta\geq 0\}$ that balance between expected reward and sensitivity, with  the weight on sensitivity increasing in $\delta$. Our eventual goal is to identify values of the parameter $\delta$ such that the corresponding solution $x_n(\delta)$ performs well out-of-sample. While parameters like $\delta$ are commonly chosen by optimizing some estimate of the out-of-sample expected reward using bootstrap or cross-validation (e.g., the regularization parameter in Lasso or Ridge regression), the characterization \eqref{eq:robust-mv-general} shows that worst-case optimization is a multi-objective problem, which suggests that estimates of both the mean and the variance (i.e., sensitivity) should be used to select $\delta$. We provide guidance for this approach by characterizing the impact of the robustness parameter on the mean and the variance of the out-of-sample reward.

\subsection*{The constrained DRO model}
An alternative to the ``penalty version" of the DRO model \eqref{eq:robust_empirical} defines set of alternative models defined in terms of a hard constraint on the deviation measure
\begin{eqnarray}
\left\{\begin{array}{l}
\begin{displaystyle}\max_x\min_{\mathbb Q} {\mathbb E}_{\mathbb Q}[f(x,\,Y)]\end{displaystyle} \\ \vspace{-0.25cm} \\ \mbox{Subject to:} \; {\mathcal H}_\phi({\mathbb Q}\,|\,{\mathbb P}_n) \leq \epsilon.
\end{array}\right.
\label{eq:constrained}
\end{eqnarray}
As in the penalty version, this leads to a family of decisions $\{x_n(\epsilon):\,\epsilon \geq 0\}$ parameterized by the constraint threshold $\epsilon$ in \eqref{eq:constrained}. With sufficient regularity the penalty and constrained versions are related via convex duality in that the solution of \eqref{eq:robust_empirical} solves \eqref{eq:constrained} for some choice of $\epsilon$ and vice versa; see Corollary 3 of \cite{ben2013} for more details. Specifically, the solution of \eqref{eq:constrained} is obtained by solving \eqref{eq:robust_empirical} with $\delta\equiv\delta_n(\epsilon)$
\begin{eqnarray}
x_n(\epsilon) = \argmax_x \Big\{\min_{\mathbb Q}{\mathbb E}_{\mathbb Q}[f(x,\,Y)] + \frac{1}{\delta_n(\epsilon)}{\mathcal H}_\phi({\mathbb Q}\,|\,{\mathbb P}_n)\Big\}
\label{eq:const-x}
\end{eqnarray}
where
\begin{eqnarray}
\delta_n(\epsilon) := \argmax_{\delta\geq 0} \Big\{\max_x \min_{\mathbb Q}{\mathbb E}_{\mathbb Q}[f(x,\,Y)] + \frac{1}{\delta}{\mathcal H}_\phi({\mathbb Q}\,|\,{\mathbb P}_n)- \frac{\epsilon}{\delta}\Big\},
\label{eq:const-delta}
\end{eqnarray}
and both models generate the same family of solutions. For this reason, we focus on developing an approach to selecting $\delta$ in the penalty formulation of the model with the understanding that it can be applied to selecting $\epsilon$ in the constrained problem.

\subsection*{DRO Caveat}

For the remainder of the paper, unless otherwise stated, all claims, findings, and insights provided are for DRO problems of the form (\ref{eq:robust_empirical}). Conditions for the $\phi$-divergence penalty and the reward function $f(x,\,Y)$ will be stated as needed.

%%%%%%%%%%%%%%%%

\section{Statistics of robust solutions: Asymptotic normality}
\label{sec:stats_robust_sol}

The results in Section \ref{sec:REO} show that robust optimization is a multi-objective problem that balances between maximizing the expected reward and minimizing its variance.
In this section, we characterize the statistical properties of the robust solution $x_n(\delta)$. These results will be used in the following sections to study the out-of-sample mean and variance (sensitivity) of the reward.

The results in this section use general results on the consistency (Theorem 5.4, \cite{SDR}) and asymptotic normality (Theorem 5.21, \cite{vdV}) of sample average optimizers, which are reproduced in Appendix \ref{sec:AN_general}. The following regularity assumptions on $f(x,\,Y)$ will generally be imposed. The results from \cite{SDR} and \cite{vdV} can be applied under more relaxed assumptions, though this does not lead to additional insights.

Let $f:{\mathbb R}^d\times{\mathbb R}^l \rightarrow {\mathbb R}$, $Y$ is an  $\mathbb{R}^{l}$-valued random vector with distribution $\mathbb{P}$, and
\begin{equation}
x^{\star}(0) := \arg \max_x \,\Big\{ \mathbb{E}_{\mathbb P}\left[f(x,\,Y)\right] \Big\}
\label{eq:opt_dgm}
\end{equation}
be the maximizer of the expected reward under the population distribution $\mathbb P$.
\begin{assumption} \label{ass1}
The function $f(x,\,Y)$ and random vector $Y\sim {\mathbb P}$ are such that
\begin{itemize}
\item $f(x,\,Y)$ is strictly concave and twice continuously differentiable in $x\in\mathbb{R}^d$ for ${\mathbb P}$-almost surely every $Y\in\mathbb{R}^{l}$;
\item for each fixed $x\in\mathbb{R}^{d}$, the mappings $y\mapsto f(x,y), \nabla_{x}f(x,y), \nabla^{2}_{x}f(x,y)$, $y\in\mathbb{R}^{l}$, are measurable and all moments of the random variables $f(x,Y)$, $\nabla_{x}f(x,Y)$, $\nabla^{2}_{x}f(x,Y)$ exist;
\item there exists a solution $x^\star(0)$ of  \eqref{eq:opt_dgm}.
\end{itemize}
\end{assumption}

It follows from Assumption \ref{ass1} is that the Hessian matrix $\mathbb{E} _{\mathbb P}\big[\nabla_x^2 f (x,\,Y)\big]$ is negative definite for every $x$, and that $x^\star(0)$ is unique.
Observe too that Assumption \ref{ass1} is not required for  \eqref{eq:robust-mv-general} to hold.  Likewise, the relationship between worst-case optimization and mean-variance optimization still holds if $x$ is constrained, though constraints will affect the efficiency of the mean-sensitivity trade-off for the reward that can be achieved.

The following result characterizes the asymptotic properties of the solution of the empirical optimization problem, and follows immediately from results in \cite{SDR} and \cite{vdV} which are stated in Appendix \ref{sec:AN_general}.

\begin{proposition} \label{prop:normality_empirical}
Suppose that $f(x,\,Y)$ satisfies Assumption \ref{ass1}.
Let $x^{\star}(0)$ be the solution of the optimization problem \eqref{eq:opt_dgm}. Then the solution $x_n(0)$ of the empirical optimization problem \eqref{emp_n} is consistent, $x_n(0)\overset{P}{\longrightarrow}x^{\star}(0)$, and asymptotically normal
\begin{equation*}
\sqrt{n}\Big(x_n(0)-x^{\star}(0)\Big) \overset{D}{\longrightarrow} N\left(0,\,\xi(0)\right), \; \mbox{as} \; n\rightarrow\infty,
\end{equation*}
where the covariance matrix
\begin{equation*}
\xi(0)  =
\mathbb{V}_{\mathbb{P}}
\Big[
{\mathbb{E} _{\mathbb P}\big[\nabla_x^2 f (x^{\star}(0),\,Y)\big]}
^{-1} \nabla_x f (x^{\star}(0),\,Y)\Big].
\end{equation*}
In particular
\begin{eqnarray*}
\sqrt{n}\Big(x_n(0)-x^{\star}(0)\Big) = \sqrt{\xi(0)}Z + o_P(1),
\end{eqnarray*}
where $\sqrt{\xi(0)}$ is a $d \times d$ matrix such that $\sqrt{\xi(0)}\sqrt{\xi(0)}'=\xi(0)$ and $Z$ is a $d$-dimensional standard normal random vector.
\end{proposition}

We now consider the asymptotic distribution of the solution $x_n(\delta)$ of the robust problem \eqref{eq:robust_empirical}. The dual characterization of $\phi$-divergence implies that $x_n(\delta)$ can be obtained by solving
\begin{eqnarray}
(x_n(\delta),\,c_n(\delta)) =\arg \min_{x,\,c}\left\{c+\frac{1}{\delta}\mathbb{E}_{\mathbb{P}_n}\Big[\phi^*\Big(\delta(-f(x,\,Y)-c)\Big)\Big]\right\},
\label{eq:robust_empirical_1}
\end{eqnarray}
where
\begin{eqnarray*}
\phi^*(\zeta):=\sup_z\big\{z\zeta-\phi(z)\big\}
\end{eqnarray*}
is the convex conjugate of $\phi(z)$.
If $\psi:{\mathbb R}^d \times {\mathbb R}\rightarrow {\mathbb R}^{d+1}$ is given by
\begin{eqnarray}
\psi(x,\,c) :=
\left[
\begin{array}{c}
(\phi^*)'\big(-\delta\left(f(x,\,Y)+c\right)\big)
\nabla_xf(x,\,Y)  \\[5pt]
\displaystyle
-\frac{\phi''(1)}{\delta} \big((\phi^*)'\big(-\delta\left(f(x,\,Y)+c\right)\big) -1\big)
\end{array}
\right],
\label{eq:psi}
\end{eqnarray}
the first-order conditions for $(x_n(\delta),\,c_n(\delta))$ are\footnote{We have added a scaling constant $-\frac{\phi''(1)}{\delta} $ in the second equation. This does not affect the solution of the first-order conditions, but makes the subsequent analysis more convenient.}
\begin{eqnarray}
{\mathbb E}_{{\mathbb P}_n}\big[\psi(x,\,c)\big]=\left[
\begin{array}{c}
{\mathbb E}_{{\mathbb P}_n}\Big[(\phi^*)'\big(-\delta\left(f(x,\,Y)+c\right)\big)
\nabla_xf(x,\,Y)\Big] \\ [8pt]
\displaystyle
{\mathbb E}_{{\mathbb P}_n}\Big[-\frac{\phi''(1)}{\delta} \big((\phi^*)'\big(-\delta\left(f(x,\,Y)+c\right)\big) -1\big)\Big]
\end{array}
\right]
=
\left[
\begin{array}{c}
0 \\
0
\end{array}
\right].
\label{eq:foc_delta}
\end{eqnarray}
Similarly, let
\begin{align}
(x^{\star}(\delta),\,c^{\star}(\delta)) & =  \arg \min_{x,\,c}
\left\{c+\frac{1}{\delta}\mathbb{E}_{\mathbb{P}} \Big[\phi^*\Big(\delta(-f(x,\,Y)-c)\Big)\Big] \right\}
\label{eq:rob_dgm}
\end{align}
be the solution of the population version of the robust problem \eqref{eq:robust_empirical_1}
with first-order conditions
\begin{eqnarray}
{\mathbb E}_{\mathbb P}[\psi(x, \, c)] =
\left[
\begin{array}{c}
{\mathbb E}_{{\mathbb P}}\Big[(\phi^*)'\Big(-\delta\big(f(x,\,Y)+c\big)\Big)\nabla_x f(x,\,Y)\Big] \\[5pt]
\displaystyle
{\mathbb E}_{{\mathbb P}}\Big[-\frac{\phi''(1)}{\delta}\left((\phi^*)'\big(-\delta\left(f(x,\,Y)+c\right)\big)-1\right)\Big]
\end{array}
\right]
=
\left[
\begin{array}{c}
0 \\
0
\end{array}
\right].
\label{eq:foc_limit}
\end{eqnarray}

Let $\xi(\delta) \in {\mathbb R}^{d\times d}$,\,$\eta(\delta)\in{\mathbb R}$, and $\kappa(\delta)\in{\mathbb R}^{d\times 1}$ be the entries of the matrix
\begin{eqnarray*}
 V(\delta) := \left[\begin{array}{cc} \xi(\delta) &
\kappa(\delta) \\ \vspace{-0.25cm}\\
\kappa(\delta)' & \eta(\delta)\end{array}\right] = A^{-1} B {A^{-1}}' ,
\end{eqnarray*}
where
\begin{align*}
A & := {\mathbb E}_{\mathbb{P}}[- J_\psi (x^{\star}(\delta),\,c^{\star}(\delta))] \in {\mathbb R}^{(d+1)\times (d+1)},\\
B & := {\mathbb E}_{\mathbb P}[\psi(x^{\star}(\delta),\,c^{\star}(\delta))\,\psi(x^{\star}(\delta),\,c^{\star}(\delta))'] \in {\mathbb R}^{(d+1) \times (d+1)},
\end{align*}
and $J_\psi$ denotes the Jacobian matrix of $\psi$.

The first part of the following result shows that $(x_n(\delta),\,c_n(\delta))$ is consistent and jointly asymptotically normal while the second part characterizes its limiting distribution in terms of the limiting distribution of the empirical distribution $x_n(0)$ when $\delta$ is small. Consistency and asymptotic normality can be established using the general results from \cite{SDR} and \cite{vdV} (Appendix \ref{sec:AN_general}). However, both $f(x,\,Y)$ as well as $\phi(x)$ need to be sufficiently regular, so in contrast to Proposition \ref{prop:normality_empirical}, Assumption \ref{ass2} will be required in addition to Assumption \ref{ass1}.
\begin{theorem}\label{thm:normality2}
Suppose $\phi(z)$ satisfies Assumption \ref{ass2} and $f(x,\,Y)$ satisfies Assumption \ref{ass1}.
Let $(x_n(\delta),\,c_n(\delta))$ be the solution of the robust empirical optimization problem
\eqref{eq:robust_empirical_1} and $(x^{\star}(\delta),\,c^{\star}(\delta))$ solve the robust problem \eqref{eq:rob_dgm} under the population distribution $\mathbb P$. Then $(x_n(\delta),\,c_n(\delta))$ is consistent with $(x_n(\delta),\,c_n(\delta)) \overset{P}{\longrightarrow} (x^{\star}(\delta),\,c^{\star}(\delta))$ and jointly asymptotically normal with
\begin{eqnarray}
\sqrt{n} \Big(x_n(\delta)-x^{\star}(\delta)\Big) & \overset{D}{\longrightarrow} & N(0,\,\xi(\delta)), \nonumber \\
\sqrt{n} \Big(c_n(\delta)-c^{\star}(\delta)\Big) & \overset{D}{\longrightarrow} & N(0,\,\eta(\delta)),
\label{eq:AN}\\
n \, \mathrm{Cov}_{\mathbb{P}}\Big[x_n(\delta),\,c_n(\delta)\Big] & \longrightarrow &  \kappa(\delta), \nonumber
\end{eqnarray}
as $n\rightarrow\infty$.
In particular,
\begin{eqnarray*}
\sqrt{n} \Big(x_n(\delta)-x^{\star}(\delta)\Big)= \sqrt{\xi(\delta)} Z + o_P(1)
\end{eqnarray*}
where $\sqrt{\xi(\delta)}$ is a $d \times d$ matrix such that $\sqrt{\xi(\delta)}\sqrt{\xi(\delta)}'=\xi(\delta)$ and $Z$ is a $d$-dimensional standard normal random vector.
Furthermore
\begin{align}
x^{\star}(\delta) & = x^{\star}(0) + \pi \delta + o(\delta),\label{eq:rob_asymp_bias}\\
c^{\star}(\delta) & = - {\mathbb E}_{\mathbb P} [f(x^{\star}(0),\,Y)]  + O(\delta),\nonumber
\end{align}
where
\begin{align}
\pi & := \frac{1}{\phi^{''}(1)} \, \Big({\mathbb{E}_{\mathbb{P}}[\nabla^2_x f(x^{\star}(0),\,Y)]}\Big)^{-1} \mathrm{Cov}_{\mathbb{P}}\Big[{\nabla_x f(x^{\star}(0),\,Y)}, \, f(x^{\star}(0),\,Y)\Big],
\label{eq:pi}
\end{align}
and
\begin{eqnarray*}
V(\delta) = \left[\begin{array}{cc} \xi(\delta) &
\kappa(\delta) \\ \vspace{-0.25cm}\\
\kappa(\delta)' & \eta(\delta)\end{array}\right]  = \left[\begin{array}{cc} \xi(0) &
\kappa(0) \\ \vspace{-0.25cm}\\
\kappa(0)' & \eta(0)\end{array}\right] + O(\delta),
\end{eqnarray*}
where
\begin{align*}
\xi(0) & =   {\mathbb V}_{\mathbb P}\Big[\Big({\mathbb E}\big[\nabla^2_x f(x^{\star}(0),\,Y)\big]\Big)^{-1}{\nabla_x f(x^{\star}(0),\,Y)}\Big], \\
\eta(0) & =  {\mathbb V}_{\mathbb P}[f(x^{\star}(0),\,Y)],\\
\kappa(0) & =   \Big({\mathbb E}_{\mathbb{P}}[\nabla^2_x f(x^{\star}(0),\,Y)]\Big)^{-1} \mathrm{Cov}_{\mathbb{P}}\big[{\nabla_x f(x^{\star}(0),\,Y)}, \, f(x^{\star}(0),\,Y)\big].
\end{align*}
\end{theorem}

%\proof{Proof of Theorem \ref{thm:normality2}}
\begin{proof}
Consistency and asymptotic normality \eqref{eq:AN}) of $(x_n(\delta),\,c_n(\delta))$ follow from general results on the distributional properties of the solutions of sample average problems in \cite{SDR} and \cite{vdV}; see Appendix \ref{sec:AN_general} and \ref{sec:consistency}.  To prove \eqref{eq:rob_asymp_bias}, we first show that $(x^\star(\delta),\,c^\star(\delta))$ is continuously differentiable in a neighborhood of $\delta=0$, following which we use the first-order conditions \eqref{eq:foc_delta} for the optimization problem \eqref{eq:robust_empirical_1} to derive the Taylor series expansion \eqref{eq:rob_asymp_bias} and the expression for the asymptotic bias \eqref{eq:pi}. Details can be found in the Appendix \ref{sec:CLT}.
\end{proof}
%\Halmos
%\endproof

Theorem \ref{thm:normality2} shows (asymptotically) that robust optimization adds a bias $\pi$ to the empirical solution where the magnitude of the bias is determined by the robustness parameter $\delta$. It can be shown that  $\pi$  optimizes a trade-off between the loss of expected reward and the reduction in variance:
\begin{align*}
\pi & = \arg\max_\pi \Big\{ \frac{\delta^2}{2} \Big( \pi'\,{\mathbb E}_{\mathbb P}\Big[\nabla^2_x f (x^{\star}(0),\,Y)\Big] \pi - \frac{2}{\phi''(1)}
\pi'\,\mathrm{Cov}_{\mathbb{P}}\big[ f(x^{\star}(0),\,Y),\,\nabla_x f(x^{\star}(0),\,Y) \big]  \Big) \\
&
 \hspace{3em}\equiv \mathbb{E}_{{\mathbb P}}\big[f(x^{\star}(0) + \delta \pi ,Y)\big] -\mathbb{E}_{{\mathbb P}}\big[f(x^{\star}(0),Y)\big]
 \\ & \hspace{4em} - \frac{\delta}{2\phi''(1)}\Big(\mathbb{V}_{{\mathbb P}}\big[f(x^{\star}(0) + \delta \pi,\,Y)\big] - \mathbb{V}_{{\mathbb P}}\big[f(x^{\star}(0),\,Y)\big]\Big) + o(\delta^2)\Big\}.
\end{align*}

\subsection*{Variability of the solution}

It has been shown \cite{BC,BM,GCK,XCM} that various worst-case regression and classification problems are equivalent to regularized versions of these problems (i.e., regression/classification with Ridge, Lasso, etc.). One implication is that robustness induces the bias-variance trade-off associated with regularization, leading to a reduction in the variance of the {\it solution}  by shrinking it towards the origin. On the other hand, we showed in Theorem \ref{thm:normality2} that robustness adds a bias to the SAA solution in the direction $\pi$ and changes its variance from ${\xi(0)}/{n}$ to ${\xi(\delta)}/{n}$. It is natural to ask whether this corresponds to shrinkage towards the origin and a reduction in the variance of the solution (perhaps when the reward function is concave).

The following example shows that this is not the case, and that robustness can simultaneously bias the SAA solution away from the origin and increase its variance, even when the reward function is concave. More generally, \eqref{eq:robust-mv-general} shows that DRO controls the variance (sensitivity) of the reward, but is only equivalent to regularizing the solution in certain special cases.

Let
\begin{eqnarray*}
f(x,\,Y) = Y h(x)-x.
\end{eqnarray*}
where the random variable $Y$ is positive with distribution $\mathbb P$. We assume $h(x) = {x^{-a}}/{(-a)}$ ($a>0$), though we will only substitute this expression for $h(x)$ later in the example.

Given samples $Y_1,\cdots,\,Y_n\sim {\mathbb P}$ of $Y$, the solution $x_n(0)$ of the empirical problem \eqref{emp_n} is the solution of the equation
\begin{eqnarray*}
\hat{Y}_nh'(x)-1 = 0,
\end{eqnarray*}
where the random variable $\hat{Y}_n=(Y_1+\cdots+Y_n)/n$.

For concreteness, suppose that the deviation measure in the robust problem \eqref{phi-div} is the modified $\chi^2$ (see, e.g., \cite{GKL}). Writing $x_n(\delta) = x_n(0)+ \pi_n\delta + o(\delta)$ for some $\pi_n$, the first-order conditions for \eqref{eq:robust-mv-general} imply
\begin{eqnarray*}
x_n(\delta)=x_n(0) + \delta \Big(\frac{\hat{\sigma}_n^Y}{\hat{Y}_n}\Big)^2 \frac{h(x_n(0))}{h''(x_n(0))} + o(\delta),
\end{eqnarray*}
where $\hat{\sigma}_n^Y$ is the sample standard deviation of $Y$ and $\hat{\sigma}_n^Y/\hat{Y}_n$ is an estimate of its coefficient of variation.  When $h(x) = {x^{-a}}/{(-a)}$, we have
$h'(x)=x^{-a-1}$ and $h''(x) = -(a+1)x^{-a-2}$, so
\begin{eqnarray*}
x_n(\delta) =  x_n(0) + \frac{\delta}{a(a+1)}x_n(0)^2 \Big(\frac{\hat{\sigma}_n^Y}{\hat{Y}_n}\Big)^2 + o(\delta).
\end{eqnarray*}
Observe that $x_n(0)=\hat{Y}^{\frac{1}{a+1}}$.
In contrast to Lasso or Ridge, which shrinks the sample average solution towards the origin, robustness introduces a positive bias which pushes it away.

The variance of the robust solution is given by
\begin{eqnarray*}
{\mathbb V}_{\mathbb P}[x_n(\delta)] = {\mathbb V}_{\mathbb P}[x_n(0)] + \frac{2\delta}{a(a+1)} \mbox{Cov}_{\mathbb P}\Big[x_n(0),\,x_n(0)^2\Big(\frac{\hat{\sigma}_n^Y}{\hat{Y}_n}\Big)^2\Big]+ o(\delta).
\end{eqnarray*}
Therefore, if the covariance of $x_n(0)$ and $x_n(0)^2\big(\frac{\hat{\sigma}_n^Y}{\hat{Y}_n}\big)^2$ is positive, i.e., $\mbox{Cov}_{\mathbb P}[x_n(0),\,x_n(0)^2\big(\frac{\hat{\sigma}_n^Y}{\hat{Y}_n}\big)^2]>0$, robustness increases the variance of the solution.

The upper plot in Figure \ref{fig:variance} shows the distribution of $x_n(0)$ and $x_n(\delta)$ when $Y$ is exponentially distributed with mean $1$, $a=0.05$, and there are $n=10$ data points. The lower plot shows the objective function ${\mathbb E}_{\mathbb P}[f(x,\,Y)]$.  When $\delta=0.04$, the variance of the robust solution $x_n(\delta)$ is $0.8$ which is larger than the variance of the solution of the empirical problem, which is $0.1$.

\begin{figure}[h]
\begin{center}
\includegraphics[height=3in]{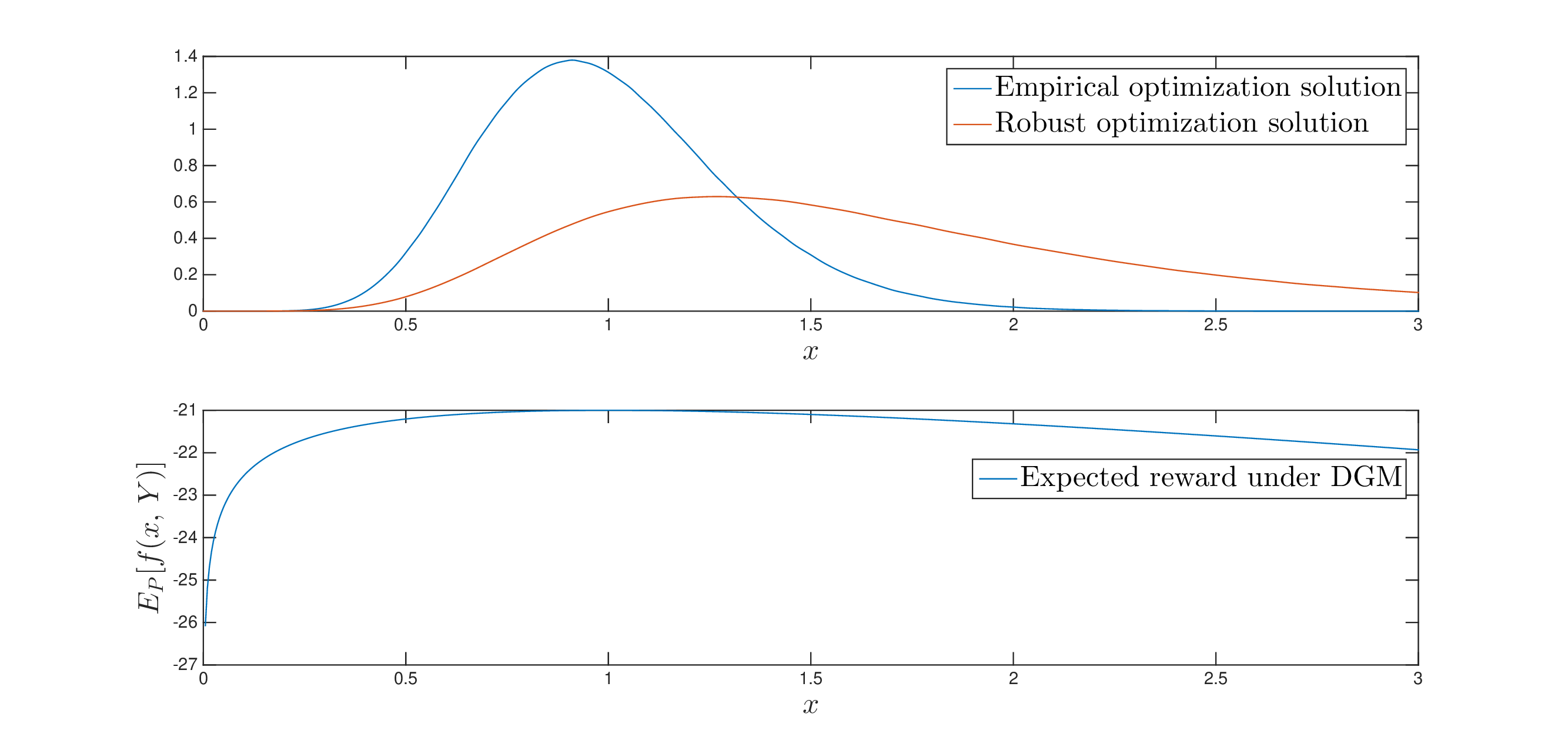}
\caption{The first plot shows the distribution of the solutions of empirical optimization and robust optimization when there are $n=10$ data points, and the second plot shows the objective function under the data generating mechanism. The optimal solution is $x^{\star}=1$. In this example, robustness adds a bias of $0.75$ to the empirical solution and pushes it away from the origin. Robustness also increases the variance of the solution from $0.1$ to $0.8$. }
\label{fig:variance}
\end{center}
\end{figure}

%%%%%%%%%%%%%%%%%%%%%%%%%%%%%%%%

\section{Out of sample performance: Empirical optimization}
\label{sec:out-of-sample-emp}

\subsection{Preliminaries}

Let ${\mathcal D} = \{Y_1,\cdots,\,Y_n\}$ be data generated i.i.d. under  $\mathbb P$,  $x_n(0)$ be the solution of the empirical optimization problem \eqref{emp_n} constructed using $\mathcal D$, and $Y_{n+1}$ be an additional sample generated under $\mathbb P$ independent of the original sample.
Since $x_{n}(0)$ depends only the data ${\mathcal D}$ and is therefore independent of $Y_{n+1}$ under the population distribution $\mathbb{P}$, it follows that $x_n(0)-x^\star(0)$ is also independent of $Y_{n+1}$ under $\mathbb P$.
Proposition \ref{prop:normality_empirical} implies
\begin{equation}
x_n(0)= x^{\star}(0) + \sqrt{\frac{\xi(0)}{n}}\big(Z + o_P(1)\big).
\label{eq:AN_for_emp_sol}
\end{equation}
In particular, $x_n(0)$ deviates from $x^{\star}(0)$ because of data variability which is captured by the term involving $Z$ and terms of $o_P(\frac{1}{\sqrt{n}})$, which for our purposes turn out to be negligible.
We study out-of-sample performance of $x_n(0)$ by evaluating the mean and variance of the reward $f(x_n(0),\,Y_{n+1})$ over random samples of the data ${\mathcal D}$ and $Y_{n+1}$ generated under $\mathbb P$.

Our subsequent results characterize the impact of robustness through a Taylor series expansion of the  mean and variance of the out-of-sample reward when the robustness parameter is small. The following result is a general statement of this expansion, which will streamline some of these derivations. The smoothness and integrability conditions are the same as those stated in Assumption \ref{ass1} and guarantee that Taylor series expansions are defined and expected values and variances exist. The proof of this result can be found in Appendix \ref{sec:TS}.

\begin{proposition}\label{prop:expansion}
Suppose that $f(x,\,Y)$ is twice continuously differentiable in $x\in\mathbb{R}^d$ for ${\mathbb P}$-almost surely every $Y\in\mathbb{R}^{l}$; that for every fixed $x\in\mathbb{R}^{d}$, the mappings $y\mapsto f(x,y), \nabla_{x}f(x,y), \nabla^{2}_{x}f(x,y)$, $y\in\mathbb{R}^{l}$, are measurable; and that all moments of the random variables $f(x,Y)$, $\nabla_{x}f(x,Y)$, $\nabla^{2}_{x}f(x,Y)$ exist where $Y$ has distribution $\mathbb P$.
Let $Y_1,\ldots,\,Y_n,\,Y_{n+1}$ be independently generated from the distribution $\mathbb P$. Let $x\in\mathbb{R}^{d}$ be constant and $\Delta\in\mathbb{R}^{d}$ a random vector independent of $Y_{n+1}$ under $\mathbb{P}$. Then for any constant $\delta$,
\begin{eqnarray}
\lefteqn{\mathbb{E}_{\mathbb P}\big[f(x+\delta \Delta,\,Y_{n+1})\big] =  \mathbb{E}_{\mathbb P}[f(x,\,Y_{n+1})] + \delta \,\mathbb{E}_{\mathbb P}[\Delta]'\,\mathbb{E}_{\mathbb P}\left[\nabla_x f(x,\,Y_{n+1})\right]} \label{eq:mean_f}
\\  [5pt]
& &     + \frac{\delta^2}{2}\mathrm{tr}\Big(\mathbb{E}_{\mathbb P}[\Delta \Delta']\,\mathbb{E}_{\mathbb P}\big[\nabla_x^2 f(x,\,Y_{n+1})\big]\Big) + o(\delta^2),
\nonumber
\\  [5pt]
\lefteqn{{\mathbb V}_{\mathbb P}\big[f(x+\delta\Delta,\,Y_{n+1})\big] = \mathbb{V}_{\mathbb{P}}\big[f(x,\,Y_{n+1})\big] +  \delta \mathbb{E}_{\mathbb P}[\Delta]'\,\nabla_x {\mathbb V}_{\mathbb P}\big[f(x,\,Y_{n+1})\big] } \label{eq:var_f}  \\[5pt]
& &  + \frac{\delta^2}{2}\, \mathrm{tr}\Big( \mathbb{E}_{\mathbb P}[\Delta\Delta']\,\nabla_x^2 {\mathbb V}_{\mathbb P}[f(x,\,Y_{n+1})]+ 2 {\mathbb V}_{\mathbb P}[\Delta]\,{\mathbb E}_{\mathbb P}[\nabla_x f(x,\,Y_{n+1})]\,{\mathbb E}_{\mathbb P}[\nabla_x f(x,\,Y_{n+1})]'\Big) + o(\delta^2), \nonumber
\end{eqnarray}
where the first and second derivatives of the variance of the reward with respect to decision $x$ satisfy
\begin{align}
\nabla_x  {\mathbb V}_{\mathbb P}\big[f(x,\,Y_{n+1})\big] & = 2\,\mathrm{Cov}_{\mathbb{P}}\big[f(x,\,Y_{n+1}),\,\nabla_x f(x,\,Y_{n+1})\big],
\label{eq:var_derivatives}
\\
\nabla_x^2 {\mathbb V}_{\mathbb P}\big[f(x,\,Y_{n+1})\big] & = 2\,{\mathbb V}_{\mathbb P}\big[\nabla_x f(x,\,Y_{n+1})\big] + 2\,\mathrm{Cov}_{\mathbb{P}}\big[f(x,\,Y_{n+1}),\,\nabla_x^2 f(x,\,Y_{n+1})\big].
\label{eq:var_derivatives2}
\end{align}
\end{proposition}

\subsection{Empirical optimization}
We now derive an expansion of the out-of-sample expected reward and variance under the empirical optimal solution, which will serve as a baseline case when studying the out-of-sample performance of robust optimization.

To ease notation, we denote the  mean and variance of the out-of-sample reward under $x_n(\delta)$ and the population solution $x^{\star}(\delta)$ as
\begin{align}
\label{eq:mv notation}
\mu_n(\delta) := {\mathbb E}_{\mathbb P}\big[f(x_n(\delta),\,Y_{n+1})\big], \qquad & v_n(\delta) := {\mathbb V}_{\mathbb P}\big[f(x_n(\delta),\,Y_{n+1})\big], \\
\mu^{\star}(\delta) := {\mathbb E}_{\mathbb P}\big[f(x^{\star}(\delta),\,Y_{n+1})\big], \qquad & v^{\star}(\delta) := {\mathbb V}_{\mathbb P}\big [f(x^{\star}(\delta),\,Y_{n+1})\big],
\nonumber
\end{align}
($\delta=0$ corresponds to empirical optimization). The curvature of the expected reward and the second derivative of the variance of the reward at $x^\star(\delta)$ also appear. We denote these as
\begin{align}
H(\delta) & := {\mathbb{E}}_{\mathbb P}\big[\nabla_x^2 f(x^{\star}(\delta),\,Y_{n+1}) \big], \nonumber \\
G(\delta) & := \nabla_x^2 {\mathbb V}_{\mathbb P}\big[f(x^{\star}(\delta),\,Y_{n+1})\big],
\label{eq:GH}
\end{align}
where $\nabla_x^2 {\mathbb V}_{\mathbb P}\big[f(x^{\star}(\delta),\,Y_{n+1})\big]$ is defined in \eqref{eq:var_derivatives2}.

The mean and variance of the out-of-sample reward depend on the variability of the empirical solution $x_n(0)$ and the properties of the reward function. The following result relates the mean and variance of the out-of-sample reward for the empirical optimizer to those of the population optimizer $x^{\star}(0)$.
\begin{proposition}
\label{prop:expansion_nominal}
Suppose that $f(x,\,Y)$ satisfies Assumption \ref{ass1}. Then the mean and variance of the out-of-sample reward under the empirical solution $x_n(0)$ satisfy
\begin{align}
\mu_n(0) &= \mu^{\star}(0) + \frac{1}{2n}\mathrm{tr}\Big(\xi(0)H(0)\Big) + o\Big(\frac{1}{n}\Big),
\label{eq:asymp_mean_nominal}
\\
v_n(0) &= v^{\star}(0) + \frac{1}{2n}\mathrm{tr}\Big({\xi(0)}\,G(0)\Big)   + o\Big(\frac{1}{n}\Big),
\label{eq:asymp_var_nominal}
\end{align}
where $\mu^{\star}(0)$ and $v^{\star}(0)$ are the out-of-sample mean and variance of the population optimizer $x^{\star}(0)$.
\end{proposition}

%\proof{Proof of Proposition \ref{prop:expansion_nominal}}
\begin{proof}
It is clear from \eqref{eq:mv notation} that the mean and variance of the out-of-sample reward depend on the distribution properties of the solution $x_n(0)$, the reward function, and the distribution of $Y_{n+1}$. For the solution,  Proposition \ref{prop:normality_empirical} shows that
\eqref{eq:AN_for_emp_sol} holds under Assumption \ref{ass1}, and hence
\begin{eqnarray*}
v_n(0) = {\mathbb V}_{\mathbb P}\big[f(x_n(0),\,Y_{n+1})\big]
 = {\mathbb V}_{\mathbb P}\Big[f\Big(x^{\star}(0)+\sqrt{\frac{\xi(0)}{n}}\big(Z + o_P(1)\big),\,Y_{n+1}\Big)\Big].
\end{eqnarray*}
This can be evaluated using Proposition \ref{prop:expansion} (with $\Delta = \sqrt{\xi(0)} Z$ and $\delta=\frac{1}{\sqrt{n}}$) (see Appendix \ref{sec:expansion_nominal} for more details), from which the expression \eqref{eq:asymp_var_nominal} follows. The expression \eqref{eq:asymp_mean_nominal} for the expected out-of-sample reward under the empirical optimal can be derived in the same way.
\end{proof}
%\Halmos
%\endproof

\begin{figure}[h]
\centering
\includegraphics[height=3in]{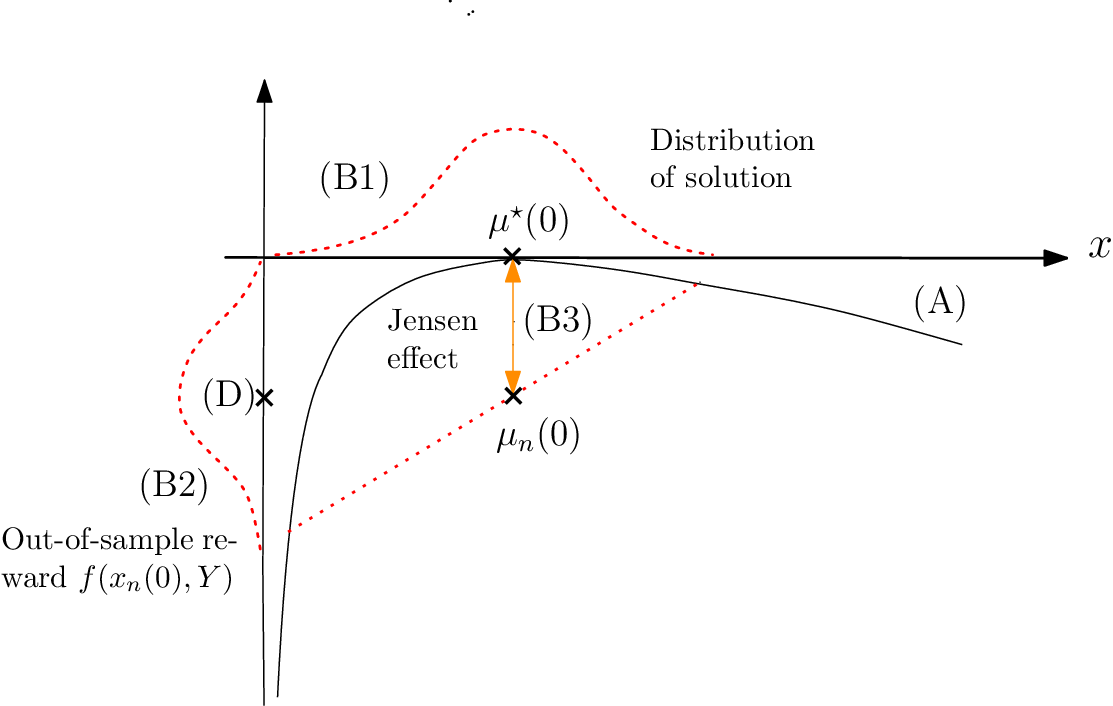}
\vspace{0.5cm}
\caption{The plot (A) is the expected reward $F(x)$ under the population distribution. (B1) is the distribution of $x_n(0)$;  (B2) is the distribution of the out-of-sample reward $f(x_n(0),\,Y)$ under  $x_n(0)$; (B3) shows the gap $\mu^\star(0)-\mu_n(0)$ from Jensen's inequality given by \eqref{eq:gap1}. The expected reward $\mu_n(0)$ is shown at (D), while $\mu^\star(0)$  is indicated on (A).}
\label{fig:Jensen1}
\end{figure}

Proposition \ref{prop:expansion_nominal} quantifies the difference between the out-of-sample means and out-of-sample variances for the solution $x^\star(0)$ of the population problem  and the empirical solution $x_n(0)$. In particular, \eqref{eq:asymp_mean_nominal} shows that the difference between the expected rewards is increasing in the variance $\frac{\xi(0)}{n}$ of the solution $x_n(0)$ and the curvature $H(0)$ of the expected reward around $x^{\star}(0)$:
\begin{eqnarray}
\mu^{\star}(0) - \mu_n(0) = - \frac{1}{2n}\mathrm{tr}\Big(\xi(0) H(0)\Big) + o\Big(\frac{1}{n}\Big) > 0.
\label{eq:gap1}
\end{eqnarray}
This performance gap can be attributed to Jensen's inequality. Specifically, if
\begin{eqnarray*}
F(x) = {\mathbb E}_{\mathbb P}\big\{f(x,\,Y_{n+1})\big|\, x\big\}
\end{eqnarray*}
denotes the out-of-sample expected reward, then $F(x)$ is strictly concave in $x$ (Assumption \ref{ass1}) and
\begin{eqnarray*}
F({\mathbb E}_{\mathbb P}[x_n(0)]) - {\mathbb E}_{\mathbb P}\big[F(x_n(0))\big] =  {\mathbb E}_{\mathbb P} \big[f({\mathbb E}_{\mathbb P}[x_n(0)],\,Y_{n+1})\big] - {\mathbb E}_{\mathbb P}\big[{\mathbb E}_{\mathbb P}\big\{f(x_n(0),\,Y_{n+1})\big|\, x_n(0)\big\}\big] >0
\end{eqnarray*}
by Jensen's inequality. Since ${\mathbb E}_{\mathbb P}[x_n(0)]=x^{\star}(0) + o(1/\sqrt{n})$
\begin{eqnarray*}
\lefteqn{F({\mathbb E}_{\mathbb P}[x_n(0)]) - {\mathbb E}_{\mathbb P}\big[F(x_n(0))\big]}\\ & = & {\mathbb E}_{\mathbb P} \big[f({\mathbb E}_{\mathbb P}[x_n(0)],\,Y_{n+1})\big] - {\mathbb E}_{\mathbb P}\big[{\mathbb E}_{\mathbb P}\big\{f(x_n(0),\,Y_{n+1})\big|\, x_n(0)\big\}\big] +  o\Big(\frac{1}{n}\Big) \\
& = &   - \frac{1}{2n}\mathrm{tr}\Big({\xi(0)}\,{\mathbb{E}}_{\mathbb P}\Big[\nabla_x^2 f(x^{\star}(0),\,Y_{n+1}) \Big]\Big)  +  o\Big(\frac{1}{n}\Big),
\end{eqnarray*}
which is precisely the performance gap in \eqref{eq:asymp_mean_nominal}  (or equivalently, \eqref{eq:gap1}).

An illustration of Proposition \ref{prop:expansion_nominal} is shown in Figure \ref{fig:Jensen1} where (A) is the expected reward $F(x)$ under the population distribution on which $\times$ is the optimal value $\mu^\star(0)$; (B1) is the distribution of $x_n(0)$ and (B2) is the distribution of the out-of-sample reward $f(x_n(0),\,Y)$. The expected out-of-sample reward $\mu_n(0)$ is marked $\times$  at (D) and  at the point labelled $\mu_n(0)$. We can see that when $F(x)$ is strictly concave, variability in $x_n(0)$ around its asymptotic mean $x^\star(0)$ leads to a gap between  $\mu^\star(0)$ and  $\mu_n(0)$ from Jensen's inequality that is shown at (B3). This gap is quantified in Proposition \ref{prop:expansion_nominal}.

\section{Out-of-sample performance: Robust optimization}
\label{sec:out-of-sample-robust}

We now study the mean and variance of the out-of-sample reward generated by solutions of robust optimization problems $x_n(\delta)$ and characterize the impact of the robustness parameter $\delta$. Our analysis consists of two parts. To begin, we derive expressions analogous to \eqref{eq:asymp_mean_nominal} and \eqref{eq:asymp_var_nominal} in which the out-of-sample mean (variance) of the reward $f(x_n(\delta),\,Y_{n+1})$ is decomposed into the mean (variance) of the reward associated with the solution $x^\star(\delta)$ of the population problem and an additional component that is determined by the variability of the solution around the asymptotic mean and the curvature of the reward function; see \eqref{eq:mean1} and \eqref{temp1}.  Next, we evaluate the impact of the robustness parameter on the mean and variance of the out-of-sample reward.

\subsection{Expected reward} \label{sec:out of sample expected reward}
We analyze the impact of robustness on the mean of the out-of-sample reward by expanding $\mu_n(\delta) = {\mathbb E}_{\mathbb P}\big[f(x_n(\delta),\,Y_{n+1})\big]$ around $\delta=0$.

To begin, we evaluate the impact of solution variability and the curvature of the expected reward on the out-of-sample expected reward under the robust solution.
Suppose that $\phi(z)$ satisfies Assumption \ref{ass2} and $f(x,\,Y)$ satisfies Assumption \ref{ass1}.
Recall from Theorem \ref{thm:normality2} that the robust solution is asymptotically normal
\begin{eqnarray}
x_n(\delta) = x^{\star}(\delta) + \sqrt{\frac{\xi(\delta)}{n}} \big(Z + o_P (1)\big).
\label{eq:rob_soln_var}
\end{eqnarray}
It follows that the out-of-sample expected reward
\begin{align}
\label{eq:mean1}
\overbrace{{\mathbb E}_{\mathbb P}\Big[f(x_n(\delta),\,Y_{n+1})\Big]}^{\small \mbox{$\mu_n(\delta)$}}
& = {\mathbb E}_{\mathbb P}\Big[f\Big(x^{\star}(\delta) + \frac{\sqrt{\xi(\delta)}}{\sqrt{n}}\big(Z+o_P(1)\big),\,Y_{n+1}\Big)\Big]  \\[5pt]
& = \underbrace{{\mathbb E}_{\mathbb P}\Big[f(x^{\star}(\delta),\,Y_{n+1})\Big]}_{\small\mbox{$\mu^{\star}(\delta)$}} + \underbrace{\frac{1}{2n}\mathrm{tr}\Big(\xi(\delta)\,{\mathbb E}_{\mathbb P}\Big[\nabla_x^2 f(x^{\star}(\delta),\,Y_{n+1})\Big]\Big)}_{\small\mbox{$\frac{1}{2n}\mathrm{tr}\big(\xi(\delta) H(\delta)\big) \equiv$ Jensen effect}} + o\Big(\frac{1}{n}\Big). \nonumber
\end{align}
Noting \eqref{eq:AN_for_emp_sol}, robustness changes the mean of the solution from $x^\star(0)$ to $x^\star(\delta)$ and its asymptotic variance from ${\xi(0)}/{n}$ to ${\xi(\delta)}/{n}$. The expression \eqref{eq:mean1} is analogous to \eqref{eq:asymp_mean_nominal} but reflects the new mean and variance of the solution.  Specifically, the first term is the expected reward under asymptotic mean $x^{\star}(\delta)$ while the second term is the Jensen effect, namely, the reduction in the expected reward
${\mathbb E}_{\mathbb P}\big[f(x^{\star}(\delta),\,Y_{n+1})\big] - {\mathbb E}_{\mathbb P}\big[f(x_n(\delta),\,Y_{n+1})\big]$
due to fluctuations of the  robust solution $x_n(\delta)$ around the asymptotic mean $x^{\star}(\delta)$ and the curvature of the expected reward.

To understand the impact of robustness on the out-of-sample expected reward, we expand both terms in \eqref{eq:mean1} around $\delta=0$ (i.e., the SAA solution). By \eqref{eq:rob_asymp_bias} and Proposition \ref{prop:expansion} (with $\Delta=\pi$) we have
\begin{align}
\label{eq:exp1}
{\mathbb E}_{\mathbb P}\Big[f(x^{\star}(\delta),\,Y_{n+1})\Big]
 & =  {\mathbb E}_{\mathbb P}\Big[f(x^{\star}(0),\,Y_{n+1})\Big] + \frac{\delta^2}{2} \pi'{\mathbb E}_{\mathbb P} \Big[\nabla_x^2 f(x^{\star}(0),\,Y_{n+1})\Big] \pi + o(\delta^2)  \\
 & =\mu^{\star}(0) + \frac{\delta^2}{2}\pi'H(0)\pi + o(\delta^2). \nonumber
\end{align}
Robustness changes the asymptotic mean of the solution $x_n(\delta)$  from $x^{\star}(0)$ to $x^{\star}(\delta)=x^{\star}(0)+\delta\pi +o(\delta)$ which changes the expected reward. However, \eqref{eq:exp1} shows that this change  is of order $\delta^2$ which is small when $\delta$ is small.

In the case of the Jensen effect we have
\begin{align}
\label{eq:jensen-robust}
\lefteqn{\frac{1}{2n}\mathrm{tr}\Big(\xi(\delta)H(\delta)\Big)}\\
& = \underbrace{\frac{1}{2n}\mathrm{tr}\Big(\xi(0)H(0)\Big)}_{\small\mbox{Jensen effect for $x_n(0)$}} + \underbrace{\frac{1}{2n}\mathrm{tr}\Big(\big(\xi(\delta)-\xi(0)\big)H(0)\Big)}_{\small \mbox{Change in variance of solution}} + \underbrace{\frac{1}{2n}\mathrm{tr}\Big(\xi(0)\big(H(\delta)-H(0)\big)\Big)}_{\small\mbox{Change in curvature}} \nonumber \\[5pt]
& + \underbrace{\frac{1}{2n}\mathrm{tr}\Big(\big(\xi(\delta)-\xi(0)\big) \big(H(\delta)-H(0)\big)\Big)}_{\small\mbox{negligible}}.\nonumber
\end{align}
The first term is the Jensen effect associated with the empirical solution \eqref{eq:asymp_mean_nominal}, while the remaining terms are adjustments to the Jensen effect for $x_n(0)$ due to the change in the variance of the solution from $\xi(0)/n$ to $\xi(\delta)/n$ and the change in the curvature due to the shift in the mean of the solution from $x^{\star}(0)$ to $x^{\star}(\delta)$. If $f(x,\,Y)$ is such that for any matrix $Q\in\mathbb{R}^{d\times d}$, the function $L:{\mathbb R^d}\rightarrow {\mathbb R}$ where
\begin{eqnarray}
\label{eq:L-def}
L(x) := \mathrm{tr}\Big(Q\,{\mathbb E}_{{\mathbb P}}\big[\nabla_x^2 f(x,\,Y_{n+1})\big]\Big)
\end{eqnarray}
is continuously differentiable, it follows from the definition \eqref{eq:GH} of $H(\delta)$ and the expansion \eqref{eq:rob_asymp_bias} of $x^{\star}(\delta)$ that
\begin{eqnarray}
\label{eq:jensen1}
\frac{1}{2n}\mathrm{tr}\Big(\xi(\delta)H(\delta)\Big) = \frac{1}{2n}\mathrm{tr}\Big(\xi(0)H(0)\Big) + \frac{\rho\delta}{2n}+ o(\delta),
\end{eqnarray}
where
\begin{eqnarray}
\rho := \mathrm{tr}\Big(\xi'(0)H(0)\Big) + \pi'\nabla_x \Big[\mathrm{tr}\Big(\xi(0)\,{\mathbb E}_{{\mathbb P}}\big[\nabla_x^2 f(x^{\star}(0),\,Y_{n+1})\big]\Big)\Big],
\label{eq:rho}
\end{eqnarray}
where $\xi'(0)$  is the derivative of $\xi(\delta)$ with respect to $\delta$ evaluated at $\delta=0$ (Theorem \ref{thm:normality2}). Compared to \eqref{eq:exp1}, where the impact of robustness is $O(\delta^2)$, the impact of robustness on the Jensen effect \eqref{eq:jensen1} is $O(\delta)$, coming from the adjustment $\frac{\rho\delta}{2n}$. This term is scaled by a factor $1/n$, so although it is linear in $\delta$, the net effect is small when $\delta$ is small. The expansion of the out-of-sample expected reward is obtained by combining \eqref{eq:mean1}, \eqref{eq:exp1}, and \eqref{eq:jensen1}.

\subsection{Variance of the reward} \label{sec:out of sample variance reward}
We analyze the impact of robustness on the variance of the out-of-sample reward by expanding $v_n(\delta) = {\mathbb V}_{\mathbb P}\big[f(x_n(\delta),\,Y_{n+1})\big]$ around $\delta=0$.

Under Assumptions \ref{ass2} and \ref{ass1}, it is shown in Theorem \ref{thm:normality2} that robustness changes the mean of the solution from $x^\star(0)$ to $x^\star(\delta)$ and its asymptotic variance from ${\xi(0)}/{n}$ to ${\xi(\delta)}/{n}$ (see also \eqref{eq:AN_for_emp_sol} and \eqref{eq:rob_soln_var}).
To see the impact on the variance of the out-of-sample reward, observe from  \eqref{eq:var_f} that
\begin{align}
\lefteqn{{\mathbb V}_{\mathbb P}\big[f(x_n(\delta),\,Y_{n+1})\big]}
\nonumber \\[5pt]
& = {\mathbb V}_{\mathbb P}\Big[f\big(x^{\star}(\delta) + \frac{\sqrt{\xi(\delta)}}{\sqrt{n}}\big(Z+o_P(1)\big),\,Y_{n+1}\big)\Big] \nonumber \\[5pt]
& = {\mathbb V}_{\mathbb P}\big[f(x^{\star}(\delta),\,Y_{n+1})\big]
\label{temp1} \\[5pt]
& + \frac{1}{2n} \mathrm{tr}\Big(\xi(\delta)\left\{\nabla_x^2 {\mathbb V}_{\mathbb P}\big[f(x^{\star}(\delta),\,Y_{n+1})\big]+ 2 {\mathbb E}_{\mathbb P}\big[\nabla_x f(x^{\star}(\delta),\,Y_{n+1})\big]{\mathbb E}_{\mathbb P}\big[\nabla_x f(x^{\star}(\delta),\,Y_{n+1})\big]'\right\}\Big) \nonumber \\[5pt]
& + o\Big(\frac{1}{n}\Big), \nonumber
\end{align}
where $\nabla_x^2 {\mathbb V}_{\mathbb P}\big[f(x^{\star}(\delta),\,Y_{n+1})\big]$ is given in \eqref{eq:var_derivatives2}. The first term in \eqref{temp1} is the variance of the reward $v^{\star}(\delta)$ at the optimal solution $x^{\star}(\delta)$ of the population robust model. The second term is the impact of the variability of the robust solution $x_n(\delta)$ on the variance of the reward. The expansion \eqref{temp1} is analogous to the expansion \eqref{eq:asymp_var_nominal} of the variance for the empirical problem, the difference being that it is for the robust solution $x_n(\delta)$ instead of $x_n(0)$.

To see the impact of robustness on the variance of the reward, we expand \eqref{temp1} around $\delta=0$.
Noting \eqref{eq:rob_asymp_bias}, it follows from Proposition \ref{prop:expansion} that the first term
\begin{eqnarray}
{\mathbb V}_{\mathbb P}\big[f(x^{\star}(\delta),\,Y_{n+1})\big] =
{\mathbb V}_{\mathbb P}\big[f(x^{\star}(0),\,Y_{n+1})\big] +  \delta \pi' \Big(\nabla_x {\mathbb V}_{\mathbb P}\big[f(x^{\star}(0),\,Y_{n+1})\big] \Big),
\label{temp2}
\end{eqnarray}
where $\nabla_x {\mathbb V}_{\mathbb P}\big[f(x^{\star}(0),\,Y_{n+1})\big]$ is defined in \eqref{eq:var_derivatives}.

As for the second term in \eqref{temp1} observe that
\begin{align*}
\lefteqn{\frac{1}{2n} \mathrm{tr}\Big(\xi(\delta)\,\nabla_x^2 {\mathbb V}_{\mathbb P}\big[f(x^{\star}(\delta),\,Y_{n+1})\big]\Big)} \\[5pt]
& = \frac{1}{2n} \mathrm{tr}\Big(\xi(0)G(0)\Big) + \frac{1}{2n}\mathrm{tr}\Big(\big(\xi(\delta)-\xi(0)\big)G(0)\Big) + \frac{1}{2n}\mathrm{tr}\Big(\xi(0)\big(G(\delta)-G(0)\big)\Big) \\[5pt]
& + \frac{1}{2n}\mathrm{tr}\Big(\big(\xi(\delta)-\xi(0)\big)\big(G(\delta)-G(0)\big)\Big).
\end{align*}
If $f(x,\,Y)$ is such that for any matrix $Q\in\mathbb{R}^{d\times d}$, the function $M:{\mathbb R^d}\rightarrow {\mathbb R}$ where
\begin{eqnarray}
\label{eq:M-def}
M(x) := \mathrm{tr}\Big(Q\,\nabla_x^2 {\mathbb V}_{{\mathbb P}}\big[f(x,\,Y_{n+1})\big]\Big)
\end{eqnarray}
is continuously differentiable, the expansion \eqref{eq:rob_asymp_bias} of $x^{\star}(\delta)$ and the definition  of $G(\delta)$ in (\ref{eq:GH}) implies
\begin{align*}
\lefteqn{\frac{1}{2n} \mathrm{tr}\Big(\xi(\delta)\,\nabla_x^2 {\mathbb V}_{\mathbb P}\big[f(x^{\star}(\delta),\,Y_{n+1})\big]\Big)} \\[5pt]
& = \frac{1}{2n} \mathrm{tr}\Big(\xi(0)G(0)\Big)+\frac{\delta}{2n}\mathrm{tr}\Big(\xi'(0)G(0)\Big)
+ \frac{\delta}{2n}\pi'\nabla_x \Big[\mathrm{tr}\Big(\xi(0)\,\nabla_x^2 {\mathbb V}_{\mathbb P}\big[f(x^{\star}(0),\,Y_{n+1})\big]\Big)\Big] \\[5pt]
& + \frac{1}{2n}O(\delta^2).
\end{align*}
Since
\begin{eqnarray*}
{\mathbb E}_{\mathbb P}\big[\nabla_x f(x^{\star}(\delta),\,Y_{n+1})\big]{\mathbb E}_{\mathbb P}\big[\nabla_x f(x^{\star}(\delta),\,Y_{n+1})\big]' \sim O(\delta^2),
\end{eqnarray*}
it follows that
\begin{eqnarray}
\frac{1}{2n} \mathrm{tr}\Big(\xi(\delta)\left\{\nabla_x^2 {\mathbb V}_{\mathbb P}\big[f(x^{\star}(\delta),\,Y_{n+1})\big]+ 2 {\mathbb E}_{\mathbb P}\big[\nabla_x f(x^{\star}(\delta),\,Y_{n+1})\big]\,{\mathbb E}_{\mathbb P}\big[\nabla_x f(x^{\star}(\delta),\,Y_{n+1})\big]'\right\}\Big) \nonumber  \\[5pt]
=  \frac{1}{2n}\mathrm{tr}\Big(\xi(0)G(0)\Big) + \frac{\theta\delta}{2n} + \frac{1}{n}O(\delta^2),
\label{eq:Jensen}
\end{eqnarray}
where
\begin{eqnarray}
\theta :=  \mathrm{tr}\Big(\xi'(0)G(0)\Big)+ \pi'\nabla_x \Big[\mathrm{tr}\Big(\xi(0)\,\nabla_x^2 {\mathbb V}_{\mathbb P}\big[f(x^{\star}(0),\,Y_{n+1})\big]\Big)\Big].
\label{eq:theta}
\end{eqnarray}
The first term in \eqref{eq:Jensen} is from the empirical optimizer \eqref{eq:asymp_var_nominal} while the term $\frac{\theta\delta}{2n}$ reflects the impact of robustness on the out-of-sample variance. An expansion of the out-of-sample variance $v_n(\delta) = {\mathbb V}_{\mathbb P}\big[f(x_n(\delta),\,Y_{n+1})\big]$ around $\delta=0$ is obtained by combining \eqref{temp2} and \eqref{eq:Jensen}.

\subsection{Main Result} \label{sec:main_results}
Let
\begin{align}
\Big\{\big(\mu_{n}(\delta),\,v_{n}(\delta)\big):\delta\geq 0\Big\}
\label{def:rmvf}
\end{align}
be the robust mean-variance frontier \eqref{eq:mv notation}. The following result characterizes the frontier in the asymptotic regime when $n$ is large and $\delta$ is small. In addition to Assumption \ref{ass1}, following assumption is required.
\begin{assumption} \label{ass3}
The function $f(x,\,Y)$ is such that for every matrix $Q\in\mathbb{R}^{d\times d}$, the functions $L:{\mathbb R}^d\rightarrow{\mathbb R}$ and $M: {\mathbb R}^d\rightarrow {\mathbb R}$ defined in (\ref{eq:L-def}) and (\ref{eq:M-def}), respectively, are continuously differentiable in $x\in{\mathbb R}^d$.
\end{assumption}
\begin{proposition} \label{prop:expansion_robust}
Suppose $\phi(z)$ satisfies Assumption \ref{ass2} and $f(x,\,Y)$ satisfies Assumptions \ref{ass1} and \ref{ass3}.
Then the expected value  of the out-of-sample reward under the robust solution $x_n(\delta)$ is
\begin{align}
\label{eq:asymp_mean_robust}
\mu_n(\delta) & = \mu^{\star}(0) + \frac{1}{2n}\mathrm{tr}\Big(\xi(0)H(0)\Big) +   \frac{\delta^2}{2[\phi^{''}(1)]^2} \beta'H(0)^{-1}\beta+ \frac{\rho\delta}{2n} \\
& + \frac{1}{n} O(\delta^2) + o\Big(\frac{1}{n}\Big) + o(\delta^2), \nonumber
\end{align}
where $H(0)$ is the curvature of the out-of-sample expected reward at $x^{\star}(0)$ as defined in \eqref{eq:GH} and
\begin{eqnarray}
\beta  := \mathrm{Cov}_{\mathbb{P}}\big[ \nabla_xf(x^{\star}(0),\,Y_{n+1}),\,f(x^{\star}(0),\,Y_{n+1}) \big].
\label{eq:beta}
\end{eqnarray}
The variance of the out-of-sample reward is
\begin{align}
\label{eq:asymp_var_robust}
v_n(\delta) & = v^{\star}(0) + \frac{1}{2n} \mathrm{tr}\Big(\xi(0) G(0)\Big)+ \frac{2\delta}{\phi^{''}(1)}\beta'H(0)^{-1}\beta
+ \frac{\theta\delta}{2n}  \\[5pt]
& + \frac{1}{n}O(\delta^2) + o\Big(\frac{1}{n}\Big) + O(\delta^2). \nonumber
\end{align}
\end{proposition}

%\proof{Proof of Proposition \ref{prop:expansion_robust}}
\begin{proof}
It follows from \eqref{eq:mean1}, \eqref{eq:exp1}, and \eqref{eq:jensen1} and the expression \eqref{eq:pi} for $\pi$ that the out-of-sample expected reward is given by \eqref{eq:asymp_mean_robust}. The out-of-sample variance is given by \eqref{temp1}, \eqref{temp2}, and \eqref{eq:Jensen}. 
\end{proof}
%\Halmos
%\endproof

For some perspective on Proposition \ref{prop:expansion_robust}, note that the mean and variance of the reward under the solution $x^{\star}(0)$ of the robust population problem are given by
\begin{align}
\mu^{\star}(\delta) & = \mu^{\star}(0) +
\frac{\delta^2}{2[\phi^{''}(1)]^2} \beta'H(0)^{-1}\beta + o(\delta^2),  \nonumber  \\[5pt]
v^{\star}(\delta) & = v^{\star}(0) + \frac{2\delta}{\phi^{''}(1)}\beta'H(0)^{-1}\beta
\label{eq:popfrontier}
+ O(\delta^2),
\end{align}
(we also get these equations when we let $n\rightarrow\infty$ in (\ref{eq:asymp_mean_robust}) and (\ref{eq:asymp_var_robust})) from which it follows that
\begin{eqnarray*}
\frac{\mathrm{d}\mu^{\star}}{\mathrm{d}\delta}(0) & = & 0, \\ [5pt]
\frac{\mathrm{d}v^{\star}}{\mathrm{d}\delta}(0)   & = & \frac{2}{\phi^{''}(1)} \beta'H(0)^{-1}\beta'.
\end{eqnarray*}
Since $f(x,\,Y)$ is strictly concave in $x$, the curvature  $H(0)$ is negative definite  so variance decreases linearly around $\delta=0$. In fact, \eqref{eq:popfrontier} shows that both the mean and variance are decreasing in $\delta$ but variance reduction is an order of magnitude larger than the reduction in the mean when $\delta$ is small ($O(\delta)$ vs. $O(\delta^2)$).  As $\delta$ increases, however, the quadratic term becomes significant and the mean decreases at an increasing rate. This suggests that the choice of $\delta$ has a significant impact on the out-of-sample reward, and that calibration of $\delta$ should therefore account for both the mean and the variance.

When $n$ is finite, the robust mean-variance frontier $\big\{(\mu_n(\delta),v_n(\delta)):\delta\geq 0\big\}$ deviates from the population frontier \eqref{eq:popfrontier} because the solution $x_{n}(\delta)$ is a random variable.  Indeed, \eqref{eq:asymp_mean_robust} and \eqref{eq:asymp_var_robust} can be written
\begin{align*}
\mu_n(\delta) & = \mu^{\star}(\delta) + \overbrace{\frac{1}{2n} \mathrm{tr}\Big(\xi(0)\, H(0) \Big) + \frac{\rho\delta}{2n}}^{\small\mbox{Jensen effect}}
 + \overbrace{\frac{1}{n} O(\delta^2) + o\Big(\frac{1}{n}\Big) + o(\delta^2)}^{\small \mbox{higher order terms}},  \\[5pt]
 v_n(\delta ) & = v^{\star}(\delta) + \frac{1}{2n} \mathrm{tr}\Big(\xi(0)\, G(0) \Big) + \frac{\theta\delta}{2n} +\frac{1}{n}O(\delta^2) + o\Big(\frac{1}{n}\Big) + O(\delta^2),
\end{align*}
showing that difference in out-of-sample expected rewards $\mu_n(\delta)-\mu^{\star}(\delta)$ comes from the variability in $x_n(\delta)$ via the Jensen effect, with a similar interpretation for the variance. The first term in the Jensen effect is from the variability of $x_{n}(0)$ around $x^{\star}(0)$ (see \eqref{eq:asymp_mean_nominal}) while the second term is the modification to the Jensen effect due to the bias that is added to the solution by robustness, which was discussed in \eqref{eq:jensen1}--\eqref{eq:rho}.

Another interpretation of Proposition \ref{prop:expansion_robust} can be obtained by writing the robust mean-variance frontier $(\mu_n(\delta), \, v_n(\delta))$ in terms the mean and variance $(\mu_n(0),\,v_n(0))$ of the out-of-sample reward for the empirical optimal $x_{n}(0)$, which was derived in Proposition  \ref{prop:expansion_nominal}.
\begin{theorem} \label{thm:expansion_robust}
Suppose $\phi(z)$ satisfies Assumption \ref{ass2} and $f(x,\,Y)$ satisfies Assumptions \ref{ass1} and \ref{ass3}.
Let the constants $\rho$ and $\theta$ be defined by \eqref{eq:rho} and \eqref{eq:theta}, and $H(0)$ and $\beta$ by \eqref{eq:GH} and \eqref{eq:beta}. Then
  \begin{align}
\label{eq:mean_rob}
\mu_n(\delta)  & = \mu_n(0) + \frac{\rho\delta}{2n}  +
\frac{\delta^2}{2[\phi^{''}(1)]^2} \beta'H(0)^{-1}\beta + \frac{1}{n} O(\delta^2)  + o(\delta^2), \\
v_n(\delta) & = v_n(0)  + \delta\Big[\frac{2}{\phi^{''}(1)}\beta'H(0)^{-1}\beta   + \frac{\theta}{2n}\Big]+ \frac{1}{n}O(\delta^2)  + O(\delta^2).
 \label{eq:var_rob}
\end{align}
\end{theorem}

In this case
\begin{align*}
\frac{\mathrm{d}\mu_n}{\mathrm{d}\delta}(0) & = \frac{\rho}{2n},  \\[5pt]
\frac{\mathrm{d}v_n}{\mathrm{d}\delta}(0)  & =  \frac{2}{\phi^{''}(1)}\beta'H(0)^{-1}\beta   + \frac{\theta}{2n},
\end{align*}
so in contrast the population frontier, derivatives are non-zero at $\delta=0$.  In the case of the variance, $\beta' H(0)^{-1}\beta < 0$. Since the term $\frac{\theta}{2n}$ is scaled by a factor $1/n$, for $n$ sufficiently large, the derivative is negative and variance reduces linearly around $\delta=0$. The rate of change of the mean is also linear but, for $n$ sufficiently large, $\frac{\rho}{2n}$ is small and variance reduction will dominate. As $\delta$ increases, the quadratic term in \eqref{eq:mean_rob} eventually dominates at which point the mean decreases at an increasing rate since the coefficient is negative. Interestingly, the marginal change in the mean is equal to the adjustment to the Jensen effect from robustness (see \eqref{eq:jensen-robust}, \eqref{eq:jensen1} and \eqref{eq:rho}).

In summary, \eqref{eq:mean_rob} and \eqref{eq:var_rob} show that when $\delta$ is small and $n$ large, the reduction in the variance of the out-of-sample reward under $x_n(\delta)$ is significantly larger than the impact on the  mean. These results apply when $x$ is unconstrained. If $x$ is constrained, the mean-sensitivity relationship \eqref{eq:robust-mv-general} still holds, so worst-case optimization is still making a trade-off between expected reward and sensitivity, but the constraints will affect the efficiency that can be achieved.

\subsection{Expected reward of DRO can outperform SAA} \label{sec:DRO_SAA}
The term $\frac{\rho}{2n}\delta$ in \eqref{eq:mean_rob} is related to the Jensen effect. In particular, robustness changes the mean of the empirical solution from $x_n(0)$ to $x_n(\delta)$ and its variance from $\xi(0)$ to $\xi(\delta)$, and it was shown in  \eqref{eq:jensen1}--\eqref{eq:rho} that this changes the Jensen effect from
\begin{eqnarray*}
\frac{1}{2n}\mathrm{tr}\Big(\xi(0) H(0)\Big)
\end{eqnarray*}
for the empirical solution $x_n(0)$ to
\begin{eqnarray*}
\frac{1}{2n}\mathrm{tr}\Big(\xi(0) H(0)\Big) + \frac{\rho}{2n}\delta.
\end{eqnarray*}
Note that  $\rho$ can be positive or negative. It is positive if the bias $\delta\pi$ that is introduced by robustness reduces the loss from the Jensen effect. When this happens,  $\mu_n(\delta)$ is larger than $\mu_n(0)$ when $\delta$ is sufficiently small for the linear term in \eqref{eq:mean_rob} to dominate the quadratic. That is, although DRO is a worst-case problem, the robust solution can have a larger out-of-sample expected reward than the SAA solution.

While there is no guarantee that $\rho$ will be positive, improvement in the expected reward under robust decisions has been documented empirically in the literature \cite{BLSZZ2013,KL,LorcaSun} and it will be seen in some of our examples. Regardless of its sign, $\frac{\rho}{2n}\delta$ is small relative to variance reduction, and disappears when $n\rightarrow \infty$.

\subsection{Summary} \label{sec:main results summary}

\begin{figure}[h]
\centering
\includegraphics[height=3in]{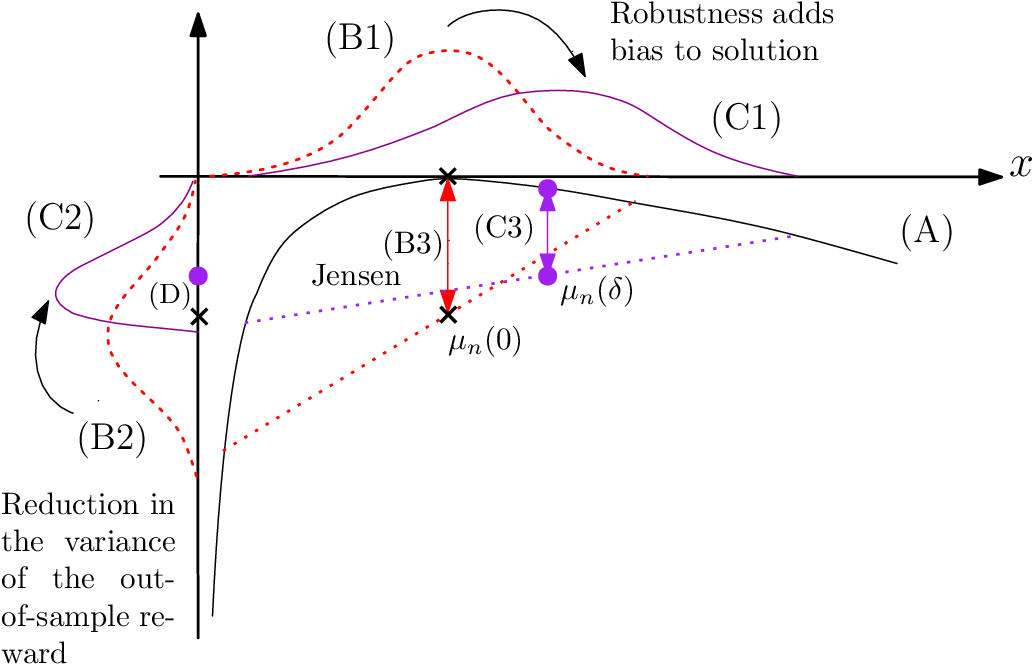}
\vspace{0.5cm}
\caption{(A) is the expected reward ${\mathbb E}_{\mathbb P}[f(x,\,Y)]$ under the population distribution. (B1), (B2), and (B3) are as described in Figure \ref{fig:Jensen1}, while (C1), (C2), and (C3) are counterparts for the robust problem. Specifically, (C1) is the distribution of the robust solution, $x_n(\delta)$; (C2) is the distribution of the out-of-sample reward $f(x_n(\delta),\,Y)$ under the robust solution; (C3) shows the Jensen effect $\mu^\star(\delta)-\mu_n(\delta)$. Indicated on (A) are the optimal rewards $\mu^\star(0)$ and $\mu^\star(\delta)$ of the population problems (see \eqref{eq:opt_dgm}, \eqref{eq:rob_dgm}, and \eqref{eq:mv notation}).}
\label{fig:Jensen2}
\end{figure}

Figure \ref{fig:Jensen2} illustrates the key theoretical insights from  this paper. Here, (A) is the expected reward ${\mathbb E}_{\mathbb P}[f(x,\,Y)]$ under the population distribution, (B1) is the distribution of the solution of the empirical problem $x_n(0)$, (B2) is the distribution of the out-of-sample reward $f(x_n(0),\, Y_{n+1})$, and (B3) is the ``Jensen effect" for the empirical problem; (B1), (B2), and (B3) were previously shown in Figure \ref{fig:Jensen1}. (C1), (C2), and (C3) are counterparts for the robust problem. Specifically, (C1) is the distribution of the robust solution, $x_n(\delta)$; (C2) is the distribution of the out-of-sample reward $f(x_n(\delta),\,Y)$ under the robust solution; (C3) shows the Jensen effect $\mu^\star(\delta)-\mu_n(\delta)$. Indicated on (A) are the optimal rewards $\mu^\star(0)$ and $\mu^\star(\delta)$ of the population problems (see \eqref{eq:opt_dgm}, \eqref{eq:rob_dgm} and \eqref{eq:mv notation}).

DRO makes a trade-off between expected reward and sensitivity. For the example in Figure \ref{fig:Jensen2}, this corresponds to adding a bias to the empirical solution that pushes $x_n(0)$ towards the flat part of the reward function. This changes the distribution of the solution from (B1) to (C1) and the distribution of the out-of-sample reward from (B2) and (C2). In particular, pushing the solution towards the flat part of the objective reduces the variance of the out-of-sample reward (compare the spread of (B2) and (C2)) and changes the mean reward from $\mu_n(0)$ to $\mu_n(\delta)$. Theorem \ref{thm:expansion_robust} quantifies the dependence of the mean and variance on $\delta$.

Robustness changes the Jensen effect from $\mu^\star(0)-\mu_n(0)$  for the SAA problem to $\mu^\star(\delta)-\mu_n(\delta)$ for the robust problem, indicated at (B3) and (C3). As indicated in the equation \eqref{eq:rho} for $\rho$, the change in the magnitude of the Jensen effect is determined by the change in the variability of the solution and the change in the curvature of the reward, which can also be see in this plot. For this example, the Jensen effect for the robust problem is smaller (i.e., $\rho>0$) and robustness improves the out-of-sample expected reward ($\mu_n(\delta)>\mu_n(0)$)  when $\delta$ is sufficiently small, which was discussed in Section \ref{sec:DRO_SAA} . This is shown at (D).

Our analysis shows that relative to the SAA solution, substantial sensitivity reduction is possible with minimal impact on the expected reward when the robustness parameter (uncertainty set) is small. These results say nothing about this trade-off when $\delta$ is large, but the benefits of robustness in all of our experiments (Section \ref{sec:applications}) are diminishing in $\delta$. This suggests that the ``small $\delta$" regime is the one that is most interesting.

\section{Calibration of Robust Optimization Models}
\label{sec:calibration}

Proposition \ref{prop:expansion_robust} shows that when the ambiguity parameter $\delta$  is small,  variance reduction is the first-order benefit of robust optimization while the impact on the mean is (almost) an order of magnitude smaller. This implies that $\delta$ should be calibrated by trading off between the out-of-sample mean and variance and not just by optimizing the mean alone, as is commonly done when tuning the free parameters (e.g., in Machine Learning applications).

Observe, however, that the robust mean-variance frontier
\begin{eqnarray*}
\Big\{\big(\mu_{n}(\delta),\,v_{n}(\delta)\big):\delta\geq 0\Big\} \equiv \Big\{\big({\mathbb E}_{{\mathbb P}}\left[f(x_n(\delta),\,Y_{n+1})\right],\,{\mathbb V}_{\mathbb P}[f(x_n(\delta),\,Y_{n+1})]\big):\delta\geq 0\Big\}
\end{eqnarray*}
cannot be computed by the decision maker because he/she does not know the data generating model $\mathbb P$, so it is natural to approximate the frontier using resampling methods. One such approach uses the well known {\it bootstrap} procedure \cite{ET}, which we now describe, and  formally state in Algorithm \ref{algo:brmvf}.

Specifically, suppose the decision maker has a data set $\mathcal{D} = \{Y_1,\cdots,\,Y_n\}$ and the associated empirical distribution ${\mathbb{P}}_{n}$. To approximate the out-of-sample behavior of different possible data sets (of size $n$), one may generate a so-called \textit{bootstrap data set}, call it $\mathcal{D}^{(1)}$, by simulating $n$ new i.i.d. data points from the empirical distribution ${\mathbb{P}}_{n}$. Associated with this bootstrap data set is the \textit{bootstrap empirical distribution} ${\mathbb{P}}^{(1)}_{n}$   \cite{ET}. This process can be repeated as many times as desired, with $\mathcal{D}^{(j)}$ and ${\mathbb{P}}^{(j)}_{n}$ denoting the bootstrap data set and bootstrap empirical distribution generated at repeat $j$. We denote the number of bootstrap samples by $k$ in Algorithm \ref{algo:brmvf}.

For each $\mathcal{D}^{(j)}$ and ${\mathbb{P}}^{(j)}_{n}$, we can compute (a family of) robust decisions $x^{(j)}(\delta)$ by solving the robust optimization problem defined in terms of the bootstrap empirical distribution ${\mathbb{P}}^{(j)}_{n}$ over a specified set of $\delta$ (step 4). The mean and variance of the reward for  $x^{(j)}(\delta)$ under the {\it original empirical distribution} ${\mathbb{P}}_{n}$, which we denote by $\mu^{(j)}(\delta)$ and $v^{(j)}(\delta)$ in steps 5 and 6, can then be computed. The $k$ bootstrap samples produces the mean-variance pairs $(\mu^{(j)}(\delta),\,v^{(j)}(\delta))$, $j=1,\cdots,\,k$. Averaging these gives an estimate of the out-of-sample mean-variance frontier (steps 7 and 8).

\begin{algorithm}[h!]
\DontPrintSemicolon
\KwIn{Data set $\mathcal{D}=\{Y_1,\ldots,Y_n\}$; ambiguity parameter grid $\mathcal{G} =
\{\delta_{1},\ldots,\delta_{m}\}$.}
\KwOut{Mean and variance of out-of-sample reward parametrized by $\delta\in\mathcal{G}$.}

   \For{$j \gets 1$ \textbf{to} $k$} {

        $\mathcal{D}^{(j)} \gets$ bootstrap data set (sample $n$ i.i.d. data points from $\mathbb{P}_{n}$) \;

	\For{$i \gets 1$ \textbf{to} $m$} {
        	$x^{(j)}(\delta_{i}) \gets \displaystyle\arg\max_x \, \min_{{\mathbb Q}} \Big\{\mathbb{E}_{{\mathbb Q}}\big[f(x,Y)\big] + \frac{1}{\delta_{i}} \mathcal{H}_{\phi}({\mathbb Q} \,|\, \mathbb{P}^{(j)}_{n})\Big\}$,

        	$\mu^{(j)}(\delta_{i}) \gets \mathbb{E}_{\mathbb{P}_{n}}\big[f(x^{(j)}(\delta_{i}),Y)\big]$,

        	$v^{(j)}(\delta_{i}) \gets \mathbb{V}_{\mathbb{P}_{n}}\big[f(x^{(j)}(\delta_{i}),Y)\big]$.
	}
    }
		
		$\hat{\mu}_n(\delta_{i}) \gets \frac{1}{k}\sum\limits^{k}_{j=1}\mu^{(j)}(\delta_{i})$, for all $\delta_i\in{\mathcal G}$\;
		
		$\hat{v}_n(\delta_{i}) \gets \frac{1}{k}\sum\limits^{k}_{j=1}v^{(j)}(\delta_{i}) + \frac{1}{k-1}\sum\limits^{k}_{j=1}\Big(\mu^{(j)}(\delta_{i}) - \hat{\mu}_n(\delta_{i})\Big)^{2}$, for all $\delta_i\in{\mathcal G}$\;

\Return{$\big\{(\hat{\mu}_n(\delta_{i}),\,\hat{v}_n(\delta_i))\,:\,i=1,...,m\big\}$}\;
\caption{{\sc Bootstrap Estimate of the Out-of-Sample Robust Mean-Variance Frontier Generated by Robust Solutions}}
\label{algo:brmvf}
\end{algorithm}

In the next section, we consider three applications, inventory control, portfolio optimization, and logistic regression. We illustrate various aspects of our theory and show how the bootstrap robust mean-variance frontier given in Algorithm \ref{algo:brmvf} can be used to effectively calibrate ambiguity parameters in such settings.
%%%%%%%%%%%%%%%%%%%%%%%%%%%%%%%%%%%

%%%%%%%%%%%%%%%%%%%%%%%%%%%%%%%%%%%
\section{Examples}
\label{sec:applications}
We consider three examples. The first is a simulation experiment of the robust inventory control problem that illustrates key elements of our theory, while the next two examples are out-of-sample tests with real data. For these experiments, the penalty function ${\mathcal H}_\phi$ is relative entropy.

\subsection{Example 1: Inventory Control}

We first consider a simulation example with
reward
\begin{eqnarray}
\label{eq:objective_newsvendor}
f(x,Y) = r \min\left\{x,Y\right\} - cx.
\end{eqnarray}
This is a so-called inventory problem where $x$ is the order quantity (decision), $Y$ is the random demand, and $r$ and $c$ are the revenue and cost parameters.

The demand distribution $\mathbb{P}$ is a mixture of two exponential distributions, ${\rm Exp}(\lambda_{L})$ and ${\rm Exp}(\lambda_{H})$, where $\lambda_{L}$ and $\lambda_{H}$ are the rate parameters. This may correspond to two demand regimes (high and low) with different demand characteristics. For this numerical example, we set the mean values as $\lambda_{L}^{-1} = 10$ and $\lambda_{H}^{-1} = 100$, and revenue and cost parameters $r = 30$ and $c = 2$. The probability that demand is drawn from the low segment is $0.7$ (or equivalently, the probability that demand is drawn from the high segment is $0.3$).

We run the following experiment. The decision maker is initially shown $n$ data points $Y_{1},\ldots,Y_{n}$ drawn \textit{i.i.d.} from the mixture distribution $\mathbb{P}$. The decision maker then optimizes the robust objective function under the empirical distribution ${\mathbb{P}}_n$ to produce the optimal robust order quantity $x^{\star}_{n}(\delta)$. Another data point $Y_{n+1}$ is then generated from $\mathbb{P}$, independent of the previous data points, and the objective value $f(x^{\star}_{n}(\delta),Y_{n+1})$ is recorded. The out-of-sample mean and variance $\mathbb{E}_{\mathbb{P}}\big[f(x^{\star}_{n}(\delta),Y_{n+1})\big]$ and $\mathbb{V}_{\mathbb{P}}\big[f(x^{\star}_{n}(\delta),Y_{n+1})\big]$ are approximated by running the experiment $K$ times, where $K$ is some large number, each time using a newly generated data set $Y_{1},\ldots,Y_{n},Y_{n+1} \sim \mathbb{P}$; the sample mean and variance are computed over the $K$ repeats.

In Figure \ref{fig:OSfrontiers}, we plot the pair $(\mathbb{E}_{\mathbb{P}}\big[f(x^{\star}_{n}(\delta),Y_{n+1})\big], \mathbb{V}_{\mathbb{P}}\big[f(x^{\star}_{n}(\delta),Y_{n+1})\big])$ for different sample sizes $n = 10, 30, 50$
with the marks on a given line correspond to different values of $\delta$; the right-most mark corresponds to $\delta = 0$ (empirical). We also plot the ``true" robust mean-variance frontier, i.e., the pair $(\mathbb{E}_{\mathbb{P}}\big[f(x^{\star}(\delta),Y)\big], \mathbb{V}_{\mathbb{P}}\big[f(x^{\star}(\delta),Y)\big])$, which is independent of $n$.

Consistent with Theorem \ref{thm:expansion_robust}, we observe significant out-of-sample sensitivity (variance) reduction when $\delta$ is small with minimal impact on the out-of-sample expected reward. Our theory says nothing about the impact of large values of $\delta$, though we see that the benefit of robustness is diminishing as $\delta$ increases, with the rate of sensitivity reduction diminishing and the expected reward being reduced at an increasing rate, as $\delta$ increases.

\begin{figure}[!h]
\centering
\includegraphics[scale=0.45]{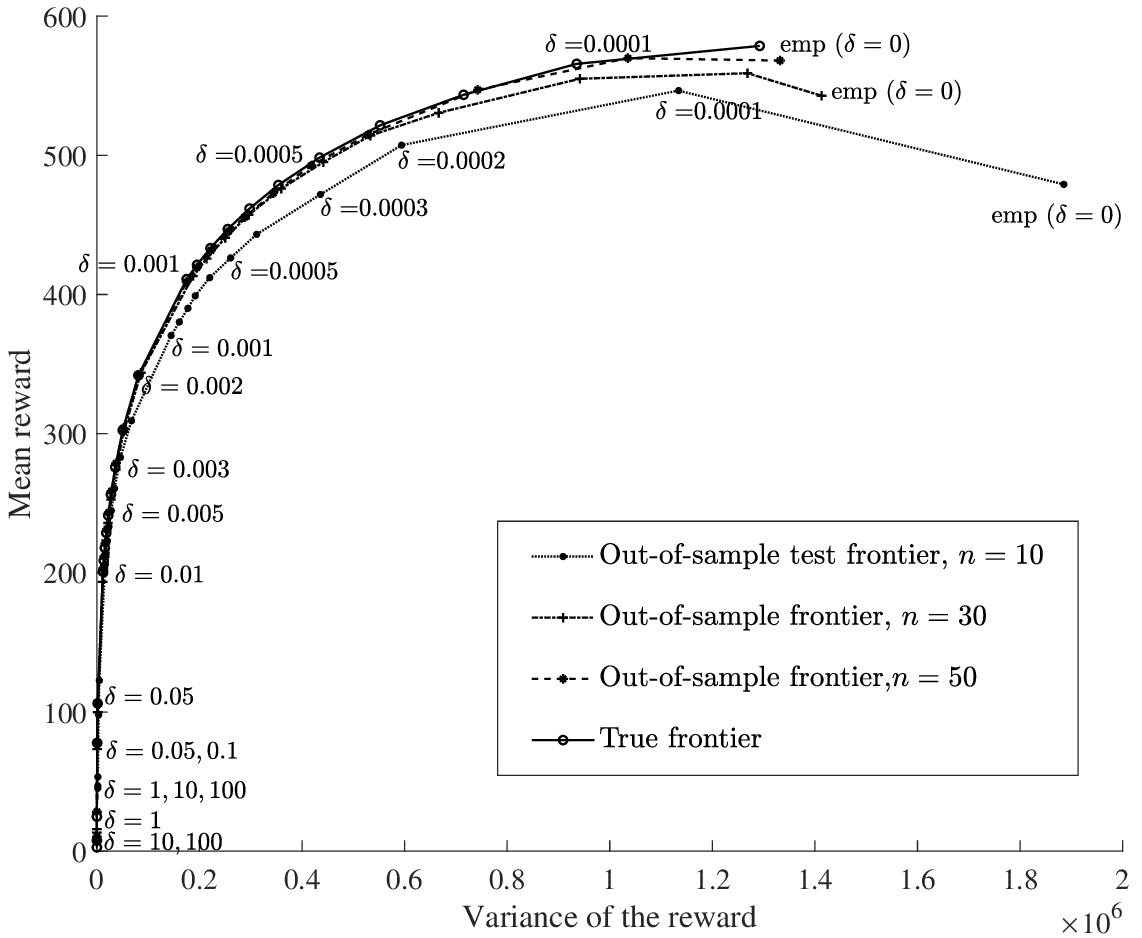}
\vspace{0.25cm}
\caption{Out-of-sample robust mean-variance frontiers for $n=10, 30$ and $50$ data points, and the true frontier generated by solutions of DRO under the data generating model for different values of the robustness parameter $\delta$.}
\label{fig:OSfrontiers}
\end{figure}

Figure \ref{fig:OSfrontiers} shows that as the sample size $n$ increases, the gap between the out-of-sample mean-variance frontiers and the ``true" robust mean-variance frontier gets smaller. This gap can be explained by Proposition \ref{prop:expansion_robust} which shows that the difference between these  frontiers should go to zero like $O(n^{-1})$. This is shown in Figure \ref{fig:maximum_frontier_gap} which plots the maximum gap (over $\delta$) between the out-of-sample frontier and the true frontier under the true frontier for different values of $n$.

\begin{figure}[!h]
\centering
\includegraphics[scale=0.4]{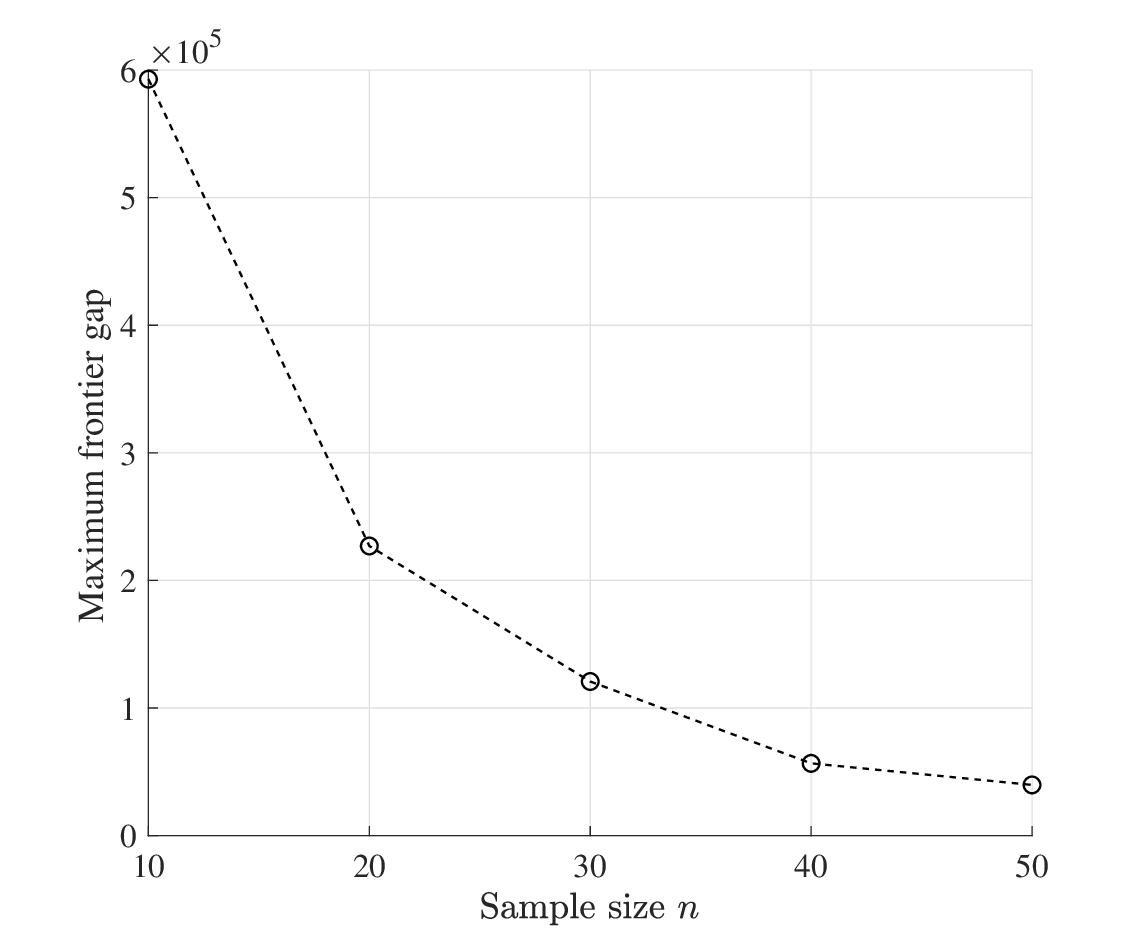}
\vspace{0.5cm}
\caption{Gap between the out-of-sample frontier and true frontier as $n$ gets large.}
\label{fig:maximum_frontier_gap}
\end{figure}

In the previous section, we proposed the bootstrap frontier as an approximation to the out-of-sample frontier. We next investigate the quality of this approximation for various sample sizes $n$. From the perspective of calibration, it is not necessary for the bootstrap frontier to equal the out-of-sample frontier, it need only preserve the \textit{relative shape}. For example, if the only difference between the bootstrap frontier and the out-of-sample frontier was that it was double its size (in both the mean and variance), the choice of $\delta$ should be the same when using either frontier since the relative trade-off between mean and variance is identical. In light of this observation, in Figures \ref{fig:OSvBoot_n=10_30_50} (A), (B), and (C) we plot the \textit{normalized} bootstrap and out-of-sample frontiers, where  the change in the mean and variance has been normalized to 1, for various sample sizes $n = 10, 30, 50$, respectively.
It is clear that as the sample size increases, the bootstrap frontier more closely approximates the relative shape of the out-of-sample frontier.

\begin{figure}
[h]
        \centering
        \begin{subfigure}[ht]{0.47\textwidth}
               \includegraphics[width=7cm]{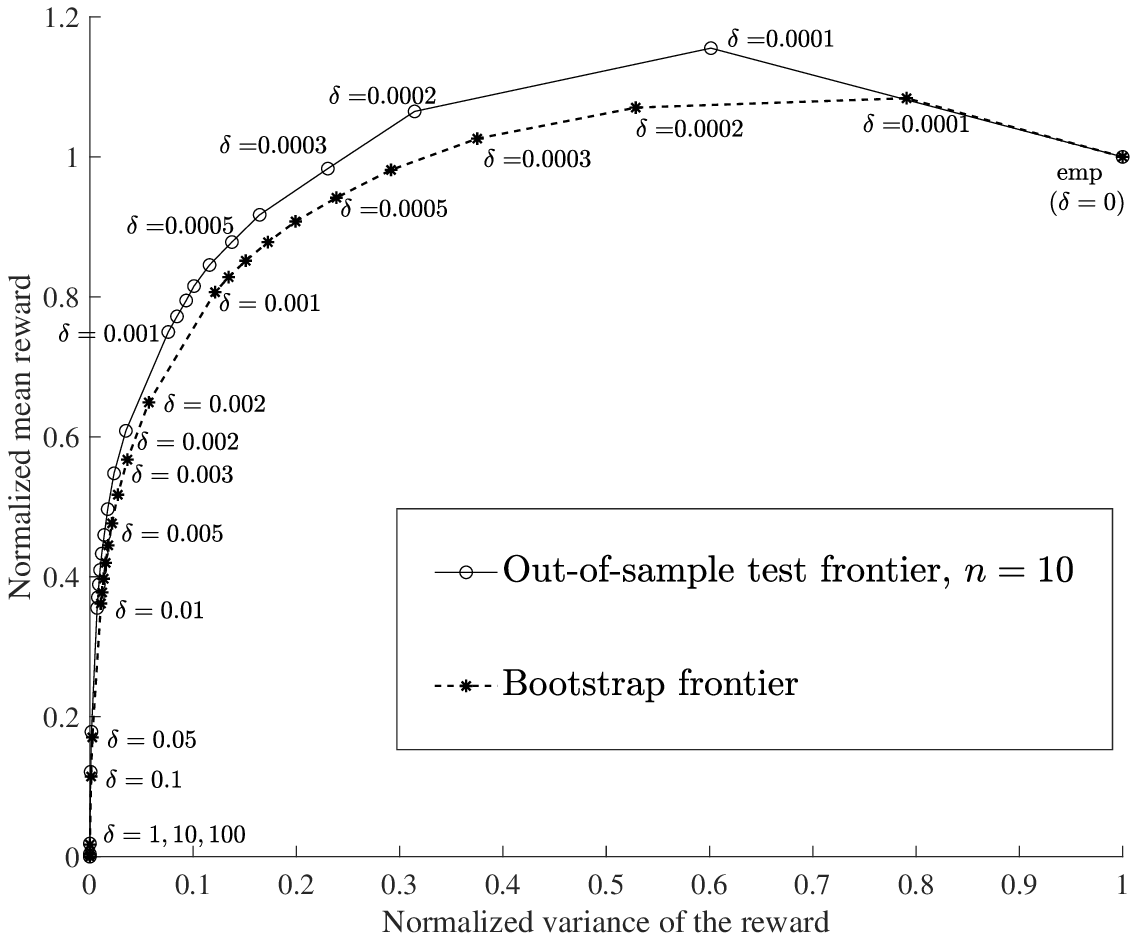}
                \caption{$n = 10$.}
        \end{subfigure}
       \hspace{0.5cm}
        \begin{subfigure}[ht]{0.47\textwidth}
                \includegraphics[width=7cm]{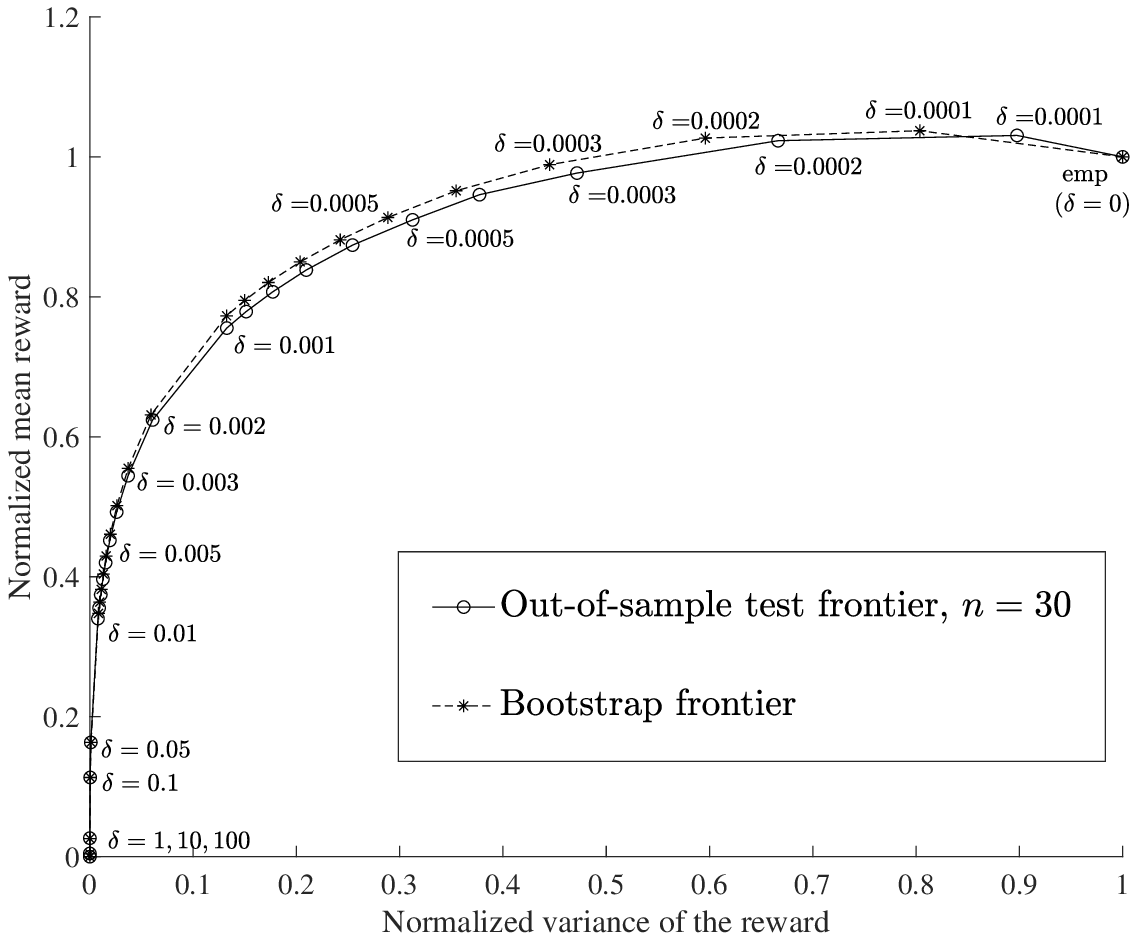}
                \caption{$n = 30$.}
        \end{subfigure}
        \vspace{0.5cm}
						
        \begin{subfigure}[ht]{0.47\textwidth}
                \includegraphics[width=7cm]{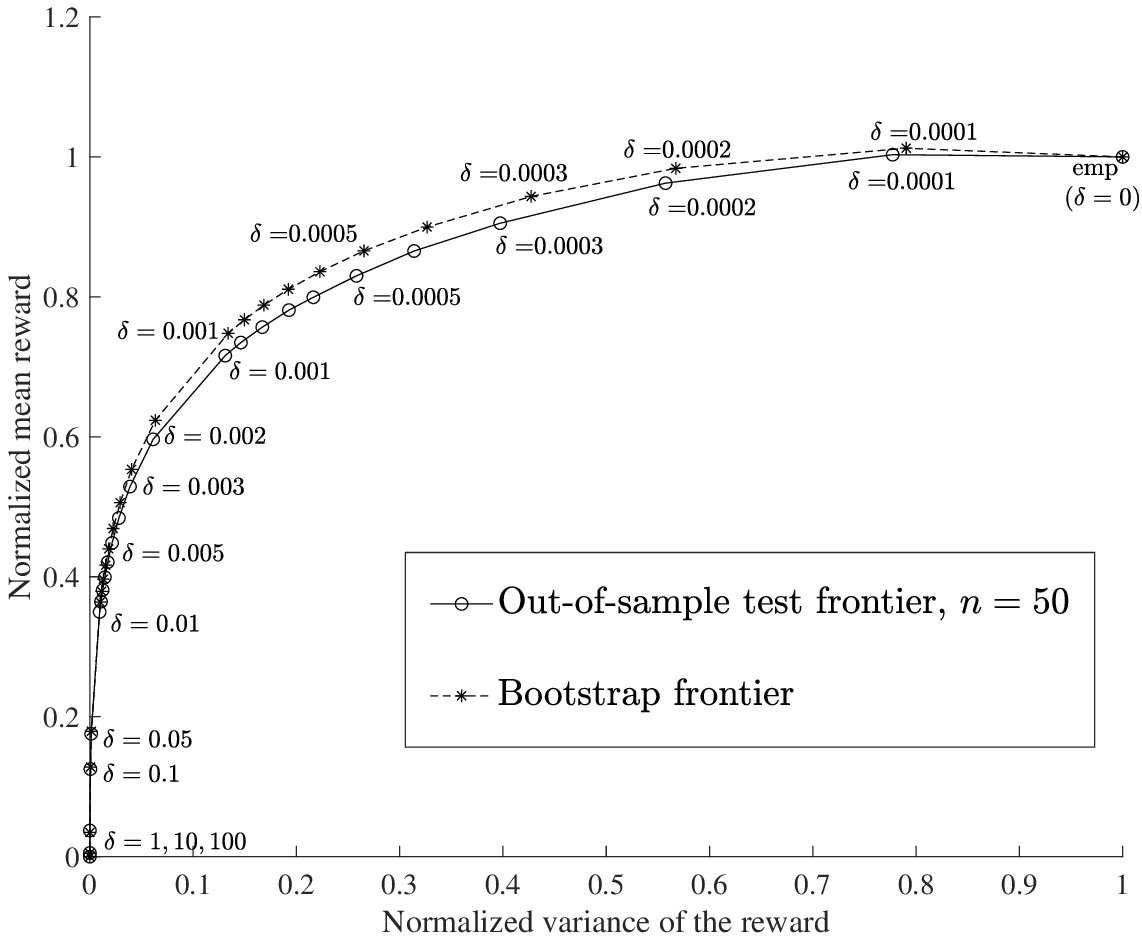}
                \caption{$n = 50$.}
                \end{subfigure}
\vspace{0.5cm}
\caption{Bootstrap frontier vs. out-of-sample frontier (with normalization). Both frontiers are scaled and normalized so that both the mean and variance equal $1$ when $\delta=0$ (i.e., empirical optimization) and $0$ in the most robust case ($\delta=100$). We see that as $n$ increases, the points on the frontier corresponding to the same values of $\delta$ converge.}
\label{fig:OSvBoot_n=10_30_50}
\end{figure}

\subsection{Example 2: Portfolio Optimization}
\label{subsec:app2}

In our second application, we consider real monthly return data from the ``10 Industry Portfolios'' data set of \cite{French}. The reward function is exponential utility of returns
\begin{eqnarray}
f(x,R) = -\exp\left(-\gamma R'x\right),
\label{utility}
\end{eqnarray}
where $x\in\mathbb{R}^{d}$ is the portfolio vector (decision variables), $R\in\mathbb{R}^{d}$ is the vector of random returns, and $\gamma$ is the risk-aversion parameter. To simplify the experiments, we choose a risk-aversion parameter $\gamma = 1$. For the purposes of this example, we impose a budget constraint $1' x = 1$ and assume that asset holdings are bounded, $-1\leq x_{i}\leq 1$, $i=1,...,d$.

We conduct the following experiment. We have $d = 10$ assets and are interested to see how robust optimization and our approach for calibrating $\delta$ perform when we estimate the $10$-dimensional joint distribution with relatively few data points ($n = 50$, for the time period April 1968 to June 1972). The robust portfolios will be tested on the empirical distributions for monthly returns of future 50 month windows, July 1972 to September 1976,  October 1976 to December 1980 and January 1981 to March 1985 that do not overlap with the training set.

To begin, we solve the robust portfolio choice problem using the 50 monthly returns for the period April 1968 to June 1972 for different values of $\delta$ and construct the robust mean-variance frontier using the bootstrap procedure described in  Algorithm \ref{algo:brmvf}. Figure \ref{fig:bootstrap_rmvf} shows this frontier, around which we also mark the $+/-$ one standard deviations of the bootstrap samples in both the mean and variance dimensions. Empirical optimization corresponds to the point $\delta = 0$, and as predicted by Theorem \ref{thm:expansion_robust}, there is significantly more reduction in the variance as compared to the mean when $\delta$ is close to $0$. Our theory does not apply when $\delta$ is large. However, it shows that significant out-of-sample sensitivity reduction with minimal impact on the mean should be expected when $\delta$ is small, while Figure \ref{fig:bootstrap_rmvf} suggests that the benefits of robustness are diminishing in $\delta$, with the rate of sensitivity (variance) reduction decreasing and the rate of mean reduction increasing at $\delta$ increases\footnote{While our results apply in the regime when $\delta$ is small, we believe that similar methods can be used to derive expansions of the solution and out-of-sample reward in orders of $\eta\equiv 1/\delta$ around the solution of the DRO problem when $\eta=0$ ($\delta=\infty$), and to study the out-of-sample performance of DRO solutions in this regime. Our experiments suggest that this is the less interesting part of the robust mean-variance frontier.}.

\begin{figure}[ht]
\centering
\includegraphics[scale=0.3]{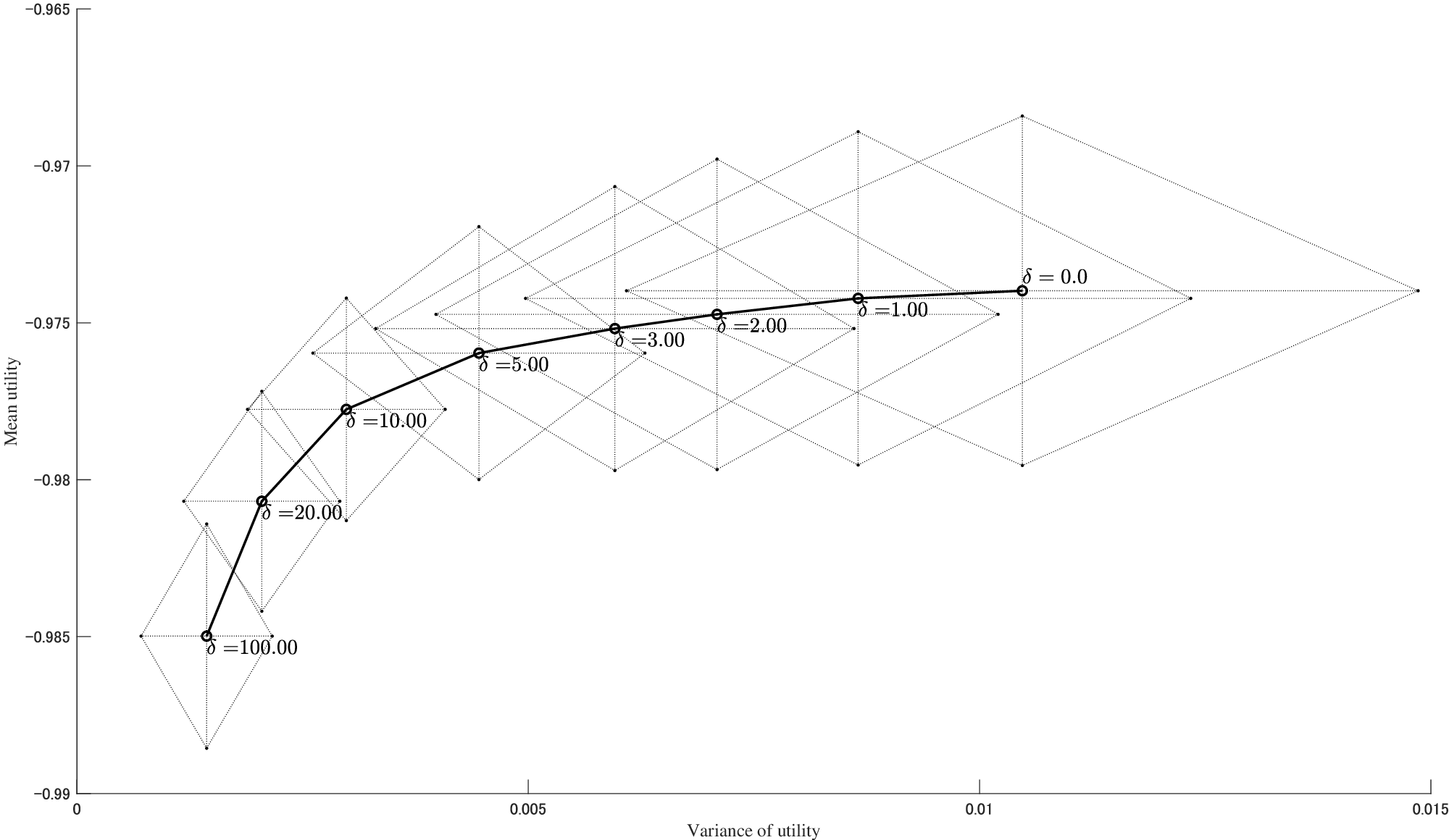}
\caption{Bootstrap robust mean-variance frontier for portfolio optimization generated using 50 months of monthly return data between April 1968 and June 1972.}
\label{fig:bootstrap_rmvf}
\end{figure}

\noindent \textit{Calibration and Out-of-sample Tests:}
The bootstrap frontier estimates the out-of-sample mean and variance for different decisions and can be used to calibrate $\delta$. For this example, the rate of variance reduction relative to the loss in the mean is substantial for $\delta\leq 5$, but this begins to diminish (and the cost of robustness increases) once $\delta$ exceeds $5$. A value of $\delta$ between $2$ and $5$ seems reasonable. While values of $\delta > 10$ may be preferred by some, the balance clearly tips towards loss in mean reward relative to variance reduction/robustness improvement.

Note also that a classical approach to calibration which optimizes the bootstrap estimate of the expected reward and ignores objective variability would select $\delta = 0$. This corresponds to empirical optimization and completely nullifies all the benefits of the robust model. More generally, while it is possible for DRO to produce solutions that out-perform empirical optimization in terms of the out-of-sample expected reward, as proven in Theorem \ref{thm:expansion_robust} such improvements are in general small relative to that of the variance reduction (when $\delta$ is small), and such solutions cannot be guaranteed even when model uncertainty is substantial, as shown by the portfolio choice example.

It is also interesting to compare calibration using the bootstrap frontier with the decisions obtained by solving the robust optimization problem
\begin{eqnarray}
\label{eq:robustSAA}
\max_x \, \min_{{\mathbb Q}\in\mathcal{U}_{\alpha}
} \left\{\mathbb{E}_{{\mathbb Q}}\big[f(x,Y)\right]\big\}
\end{eqnarray}
where the threshold $Q_{n}(\alpha)$ in the uncertainty set
\begin{eqnarray}
\mathcal{U}_{\alpha} = \left\{\mathbb{Q} : \mathcal{H}(\mathbb{Q}\,|\,\mathbb{P}_{n}) \leq Q_{n}(\alpha)\right\}
\end{eqnarray}
is the $(1-\alpha)$-quantile of the distribution of $\mathcal{H}(\mathbb{Q}\,|\,\mathbb{P}_{n})$ that we generate by simulating distributions ${\mathbb Q}$ by bootstrapping from ${\mathbb P}_n$. Confidence levels $1 - \alpha$ with value $90\%$, $95\%$, or $99\%$ are commonly suggested in the robust optimization literature.

It follows from duality that for any threshold $Q_n(\alpha)$ there is a unique ambiguity parameter value $\delta=\delta_{\alpha} > 0$, for which the solution of \eqref{eq:robustSAA} coincides with our robust solution $x(\delta_{\alpha})$; see Corollary 3 in \cite{ben2013} and also Appendix \ref{sec:significance}. In Table \ref{table:alpha_delta} we report the ambiguity parameter values $\delta_{\alpha}$ for significance levels $\alpha$ corresponding to different points on the robust mean-variance frontier in Figure \ref{fig:bootstrap_rmvf}.

{\footnotesize
\begin{table}[ht]
\caption{Corresponding ambiguity parameter values $\delta_{\alpha}$ for various traditional values of the significance level $\alpha$}
\centering
\begin{tabular}{cc}
\\
\hline
Significance level $\alpha$	 & Ambiguity values $\delta_{\alpha}$ \\
\hline
$\approx$ 1											&    $<$10 \\
0.8													& 23 \\
0.10   		 									 	& 30                                              \\
0.05   		 									 	& 31                                              \\
0.01   												 & 33                                              \\
\hline
\end{tabular}
\label{table:alpha_delta}
\end{table}
}

While robust decisions associated with the ``classical" significance levels of $\alpha=0.01, 0.05, 0.1$ in Table \ref{table:alpha_delta} may or may not perform well out-of-sample in any given application, it is clear that the range of  $\delta_\alpha$ associated with these significant levels is limited. This directly impacts the range of possible solutions available to the decision maker, which in this example are concentrated on the ``extremely conservative" region of the bootstrap  frontier (Figure \ref{fig:bootstrap_rmvf}). The values of $\delta \leq 10$ associated with the desirable part of the frontier in Figure \ref{fig:bootstrap_rmvf} correspond to significance levels $\alpha$ very close to $1$. These significance levels are difficult to estimate using simulation because the associated quantiles of ${\mathcal H}({\mathbb Q}\,|\,{\mathbb P}_n)$ are very close to $0$. While the asymptotic properties of ${\mathcal H}({\mathbb Q}\,|\,{\mathbb P}_n)$ can  be used to estimate $\alpha$ when the support of the data generating distribution $\mathbb P$ is finite and $\phi(z)$ is sufficiently regular,
this seems pointless if we treat $\alpha$ as a parametrization of the trade-off between mean and sensitivity and accept that there is nothing special about classical significance levels when it comes to out-of-sample performance.

We next analyze the out-of-sample performance of the robust solutions obtained by solving the in-sample problem with $\delta = 1, 2, 3, 5, 10$ (as suggested by the bootstrap frontier in Figure \ref{fig:bootstrap_rmvf}), the solutions of \eqref{eq:robustSAA} corresponding to $\delta_{\alpha} = 30$ ($\alpha = 0.10)$, $\delta_{\alpha} = 31$ ($\alpha = 0.05)$, $\delta_{\alpha} = 33$ ($\alpha = 0.01)$, and the solution of the empirical optimization problem (i.e., the robust solution with $\delta = 0$). We test each of the solutions on three out-of-sample test sets of data size $n = 50$ that do not overlap with the training data. Out-of-sample results are shown in Figure \ref{fig:portfolio_out-of-sample_tests}. Note that in the figure we report the mean and variance of the utility function $f(x,R)$ defined in \eqref{utility}, which is consistent with the framework of the robust optimization model in this paper.

\begin{figure}[h!]
\centering
\includegraphics[scale=0.6]{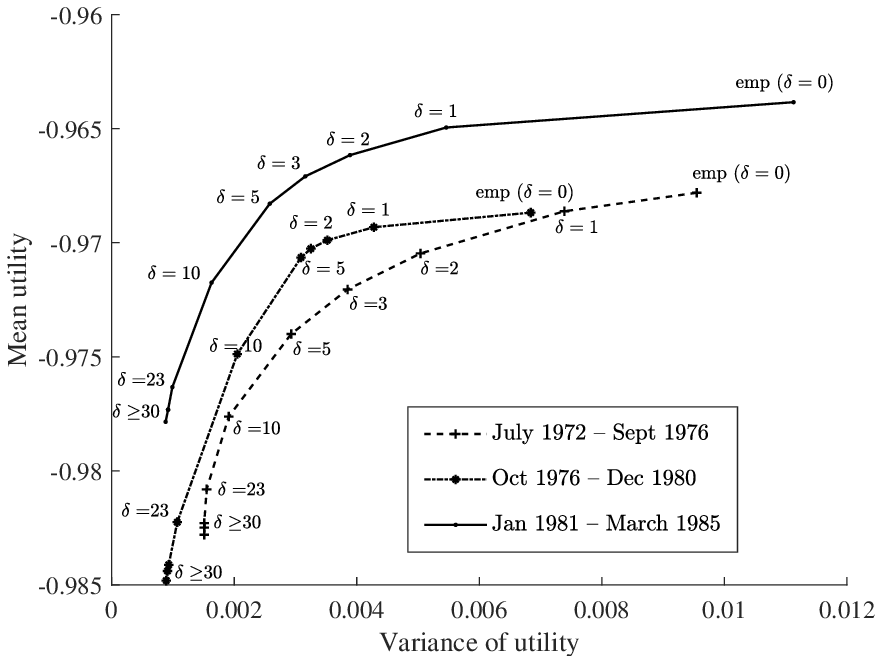}
\vspace{0.25cm}
\caption{Three out-of-sample robust mean-variance frontiers for the portfolio problem.\\
The frontiers are the average mean and variance over test data sets of 50 months.}
\label{fig:portfolio_out-of-sample_tests}
\end{figure}

Figure \ref{fig:portfolio_out-of-sample_tests}
 consistently shows that as the value of $\delta$ (``robustness") increases from zero, the mean performance degrades but the objective variability reduces, which is expected from our theory. Consistent with the bootstrap frontier (Figure \ref{fig:bootstrap_rmvf}), the portfolios associated with ``high confidence uncertainty sets" have low variance though the impact on expected utility is substantial.

\subsection{Example 3: Logistic Regression}
As final application we apply robust optimization to logistic regression which we evaluate on the WDBC breast cancer diagnosis data set \cite{Lichman}.

The reward function for logistic regression is given by
\begin{equation}
f((x,x_0),(Y,Y_0))=\ln (1 + \exp( -Y_0(x' Y+x_0) )),
\label{eq:reward_of_logistic_regression}
\end{equation}
where $Y_0\in\{-1,\,1\}$ is the binary label, $Y$ is the vector of covariates, and $x$ and $x_0$ are decision variables representing coefficients and intercept, respectively, of the linear model for classification.
Ordinary logistic regression is formulated as the maximization of the sample average of \eqref{eq:reward_of_logistic_regression}.

To demonstrate the out-of-sample behavior of robust maximum likelihood, we solve the ordinary/robust likelihood maximization problem using the first half of the WDBC breast cancer diagnosis data set \cite{Lichman}, i.e., 285 out of the 569 samples, and compute the log-likelihood and the variance of the log-likelihood of the resulting model using the remaining half of the samples, i.e., 284 out of 569. Figure \ref{fig:bsf_logistic_regression} shows both the bootstrap frontier and the frontier obtained from the out-of-sample test. Once again, choices of $\delta$ that deliver good out-of-sample log-likelihood can be obtained using the bootstrap estimate of the robust frontier with small values of $\delta$ delivering significant sensitivity (variance) reduction with minimal impact on the expected reward. As in the case of Figure \ref{fig:bootstrap_rmvf}, the benefits of robustness diminish as $\delta$ increases.

\begin{figure}[h]
\begin{center}
\includegraphics[scale=0.4]{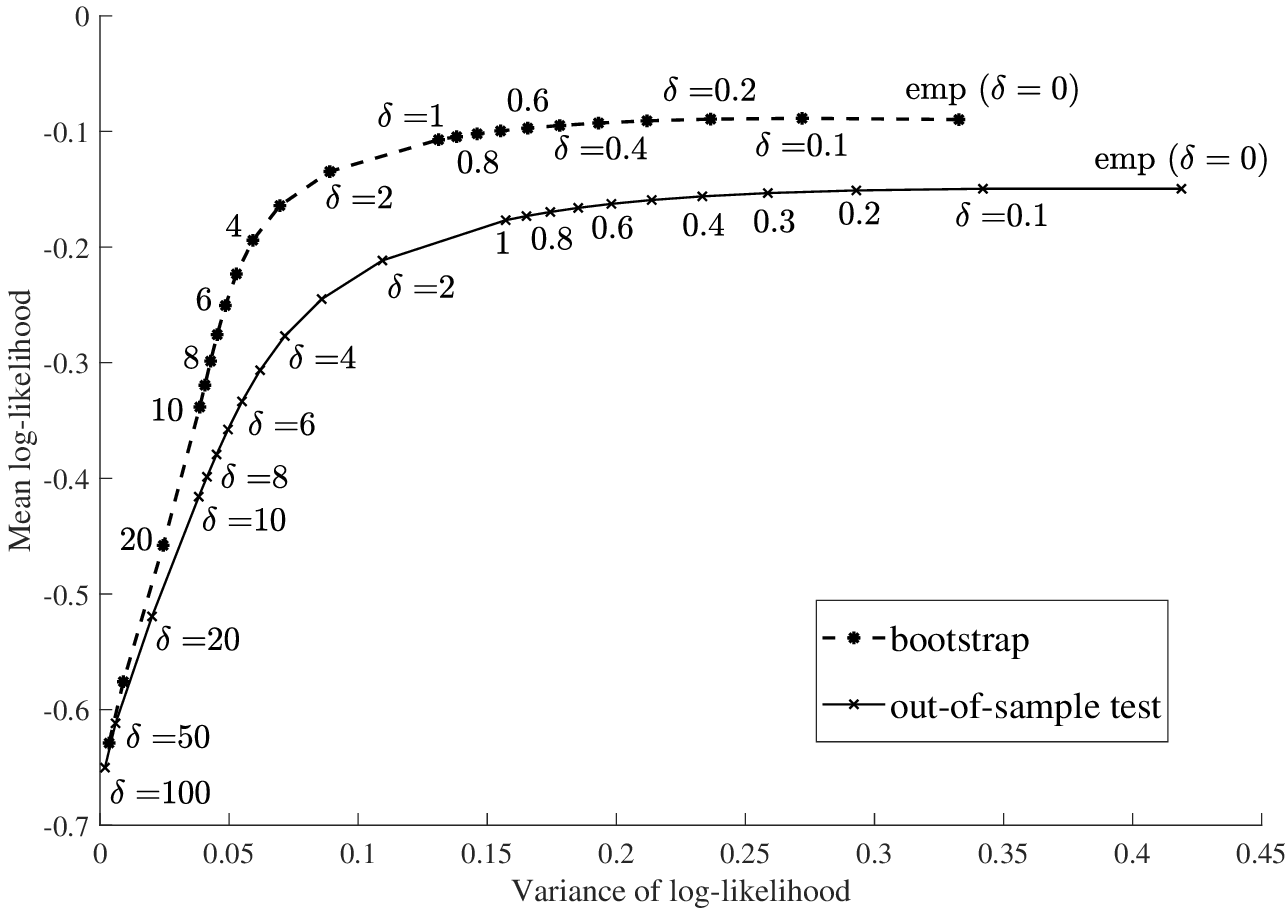}
\vspace{0.5cm}
\caption{Bootstrap frontier vs. out-of-sample test frontier with WDBC breast cancer diagnosis data set.\\
The out-of-sample test frontier shows the mean and variance of the log-likelihood of the second-half of the samples of the WDBC data set (i.e., 284 samples) on the basis of solutions obtained by using the first-half of those (i.e., 285 samples). Here, we used three attributes: no.2, no.24, and no.25 (out of 30 available covariates), which are found to be best possibly predictive in the paper \cite{MSW}. To solve the optimization problems, we used RNUOPT (NTT DATA Mathematical Systems Inc.), a nonlinear optimization solver package.}
\label{fig:bsf_logistic_regression}
\end{center}
\end{figure}

%%%%%%%%%%%%%%%%%%%%%%%%%%

%%%%%%%%%%%%%%%%%%%%%%%%%%
\section{Conclusions}
\label{sec:conclusions}

Proper calibration of DRO models requires a principled understanding of how the distribution of the out-of-sample reward depends on the ``robustness parameter." In this paper, we studied out-of-sample properties of robust empirical optimization and developed a theory for data-driven calibration of the robustness parameter.

Robustness is closely related to controlling the variance of the reward distribution as this reduces sensitivity of the mean reward to misspecification. Our main results show that the first-order benefit of ``little bit of robustness" is a significant reduction in the variance of the out-of-sample reward while the impact on the mean is almost an order of magnitude smaller. Our results imply that the robustness parameter should be calibrated by making trade-offs between estimates of the out-of-sample mean and variance.

To calibrate the robustness parameter, we introduced the robust mean-variance frontier and showed that it can be approximated using resampling methods like the bootstrap. We applied the robust mean-variance frontiers to three applications: inventory control, portfolio optimization and logistic regression. Our results showed that classical calibration methods that match ``standard" confidence levels (e.g., $95\%$) are typically associated with an excessively large robustness parameter and overly pessimistic solutions that perform poorly out-of-sample, while ignoring the variance and calibrating purely on the basis of the mean can lead to a robustness parameter that is too small and a solution that misses out on the first-order benefits of robust optimization.

%%%%%%%%%%%%%%%%%%%%%%%%%%

\subsection*{Acknowledgments}
J. Gotoh is supported in part by JSPS KAKENHI Grant 19H02379, 19H00808, and 16H01833.
M.J. Kim is supported in part by the Natural Sciences and Engineering Research Council (NSERC) Discovery Grant RGPIN-2015-04019.
A.E.B. Lim is supported in part by the Singapore Ministry of Education Academic Research Fund Tier 2 MOE2016-T2-1-086.

%\clearpage
%\newpage

%%%%%%%%%%%%%%%%%%%%%%%%%%

%%%%%%%%%%%%%%%%%%%%%%%%%%

\newpage

\appendix

\section{Asymptotics: General results}
\label{sec:AN_general}
Let
\begin{eqnarray}
x^{\star} & := &  \arg \max_{x\in{\mathcal X}} \Big\{ {\mathbb E}_{\mathbb P}\big[ f(x,\,Y)\big] \Big\}, \label{eq:appendix_pop} \\ [5pt]
x_n & := &\arg \max_{x\in{\mathcal X}} \Big\{ \mathbb{E}_{{\mathbb P}_n} \big[ f(x,\,Y)\big] \equiv \frac{1}{n}\sum_{i=1}^n f(x,\,Y_i) \Big\}.
\label{eq:appendix_saa}
\end{eqnarray}
The following result (Theorem 5.4 from \cite{SDR}) gives conditions under which $x_n$ is asymptotically consistent.
\begin{theorem}[Theorem 5.4, \cite{SDR}]
\label{thm:SDR}
Let $x_n$ and $x^\star$ be defined in \eqref{eq:appendix_pop}--\eqref{eq:appendix_saa}. Suppose that
\begin{enumerate}
\item[(i)] $f(\cdot,\,Y)$ is upper semicontinuous for $\mathbb P$-a.s. every $Y$,
\item[(ii)] $f(\cdot,\,Y)$ is concave in $x$, $\mathbb P$ almost surely,
\item[(iii)] the set $\mathcal X$ is closed and convex,
\item[(iv)]  $F(x) := {\mathbb E}_{\mathbb P}[f(x,\,Y)]$ is upper semicontinuous and there  exists a point $\bar{x}\in{\mathcal X}$ such that $F(x)>-\infty$ for all $x$ in a neighborhood of $\bar x$,
\item[(v)]  the set ${\mathcal S}$ of optimal solutions of the true problem \eqref{eq:appendix_pop} is nonempty and bounded, and
\item[(vi)] the LLN holds pointwise: For every $x\in{\mathcal X}$
\begin{eqnarray*}
\lim_{n\rightarrow\infty}\frac{1}{n}\sum_{i=1}^n f(x,\,Y_i) = {\mathbb E}_{\mathbb P}[f(x,\,Y)]
\end{eqnarray*}
\end{enumerate}
Then $x_n\rightarrow x^\star$.
\end{theorem}

Theorem 5.21 from \cite{vdV} gives conditions under which $x_n$ from \eqref{eq:appendix_saa} is asymptotically normal.
\begin{theorem}[Theorem 5.21, \cite{vdV}] \label{theorem: AN}
Let $x_n$ and $x^\star$ be defined as in \eqref{eq:appendix_pop}--\eqref{eq:appendix_saa}.  For every $x$ in an open subset of Euclidean space, let $x\mapsto
\nabla_x f(x,\,Y)$ be a measurable vector-valued function such that, for every $x^1$ and $x^2$ in a neighborhood of $x^{\star}$ and a measurable function $F(Y)$ with ${\mathbb E}_{\mathbb P} [F(Y)^2]<\infty$,
\begin{eqnarray*}
\|\nabla_x f(x^1,\,Y)-\nabla_x
f(x^2,\,Y)\| \leq F(Y)\|x^1-x^2\|.
\end{eqnarray*}
Assume that ${\mathbb E}_{\mathbb P} \|\nabla_x^2 f(x^{\star},\,Y)\|^2<\infty$ and that the map $x\mapsto {\mathbb E}_{\mathbb P}[\nabla_x f(x,\,Y)]$
is differentiable at a solution $x^{\star}$ of the equation ${\mathbb E}_{\mathbb P}[\nabla_x f(x,\,Y)]=0$ with non-singular derivative matrix
\begin{eqnarray*}
\Sigma(x^{\star}) := \nabla_x {\mathbb E}_{\mathbb P} [\nabla_x f(x^{\star},\,Y)].
\end{eqnarray*}
If ${\mathbb E}_{{\mathbb P}_n}\big[\big\|\nabla_x f(x_n,\,Y)\big\|\big] = o_P(n^{-1/2})$,
and $x_n\overset{P}{\to}x^{\star}$, then
\begin{eqnarray*}
\sqrt{n}\left(x_n-x^{\star}\right)=-\Sigma(x^{\star})^{-1}\frac{1}{\sqrt n} \sum_{i=1}^n \nabla_x f(x^{\star},\,Y_i)+o_P(1).
\end{eqnarray*}
In particular, the sequence $\sqrt{n}(x_n-x^{\star})$ is asymptotically normal with mean $0$ and covariance matrix $\Sigma(x^{\star})^{-1}{\mathbb E}_{\mathbb P}[\nabla_x f(x^{\star},\,Y)\nabla_x f(x^{\star},\,Y)'](\Sigma(x^{\star})^{-1})'$.
\end{theorem}

Under conditions that allow us to exchange the order of differentiation with respect to $x$ and integration with respect to $Y$, we have
\begin{eqnarray*}
\Sigma(x) =\nabla_x^2{\mathbb E}_{\mathbb P} [f(x,\,Y)] = \nabla_x {\mathbb E}_{\mathbb P} [\nabla_x f
(x,\,Y)] = {\mathbb E}_{\mathbb P}\big[\nabla_x^2 f (x,\,Y)\big].
\end{eqnarray*}
Note however that Theorem \ref{theorem: AN} does not require that $x\mapsto f(x,\,Y)$ is twice differentiable everywhere for $\Sigma(x)$ to exist.

\section{Proof of Proposition \ref{prop:consistent} (Consistency of solutions)}
\label{sec:consistency}

Observe that \eqref{emp_n} and \eqref{eq:robust_empirical_1} are standard empirical optimization problems, so it follows from  Theorem \ref{thm:SDR} that  $x_n(0)$ and $(x_n(\delta),\,c_n(\delta))$ are consistent if Assumptions \ref{ass2} and \ref{ass1} are satisfied.
\begin{proposition} \label{prop:consistent}
Suppose Assumptions \ref{ass2} and \ref{ass1} are satisfied. Then
$x_n(0)\overset{P}{\longrightarrow}x^{\star}(0)$ and $(x_n(\delta),\,c_n(\delta)) \overset{P}{\longrightarrow} (x^{\star}(\delta),\,c^{\star}(\delta))$.
\end{proposition}

%\proof{Proof of Proposition \ref{prop:consistent}}
\begin{proof}
We begin with $(x_n(\delta),\,c_n(\delta))$. Let
\begin{eqnarray*}
g(x, \, c, Y)
:=  c+\frac{1}{\delta}
\phi^*\Big(\delta(-f(x,\,Y)-c)\Big)
.
\end{eqnarray*}
We can write the objective functions in \eqref{eq:robust_empirical_1} and \eqref{eq:rob_dgm} as
${\mathbb E}_{{\mathbb P}_n} [g(x,\, c,\,Y)]$ and ${\mathbb E}_{{\mathbb P}} [g(x,\, c,\,Y)]$, respectively, as functions of $(x,c)$.
Since $\phi^*(\zeta)$ is convex and  non-decreasing in $\zeta$ and $\delta(-f(x,\, Y)-c)$ is jointly convex in $(x,\,c)$ for $\mathbb P$-a.e. $Y$, $g(x,\,c,\,Y)$ is jointly convex in $(x,\,c)$. It now follows from Theorem \ref{thm:SDR} (Theorem 5.4 in \cite{SDR}) that  $(x_n(\delta),\,c_n(\delta)) \overset{P}{\to} (x^{\star}(\delta),\,c^{\star}(\delta))$. Consistency of $x_n(0)$ also follows from the same result. 
\end{proof}
%\Halmos
%\endproof

\section{Proof of Theorem \ref{thm:normality2}}
\label{sec:CLT}

The following result is a direct application of Theorem \ref{theorem: AN} to the first-order conditions \eqref{eq:foc_delta} of the robust optimization problem \eqref{eq:robust_empirical_1}. Theorem \ref{thm:normality2} is essentially a refinement of Proposition \ref{prop:normality1} that expresses these results in a form that enables us to analyze the out-of-sample reward distribution under the worst-case solution.

\begin{proposition} \label{prop:normality1}
Suppose Assumptions \ref{ass2} and \ref{ass1} hold.
Let $(x_n(\delta),\,c_n(\delta))$ solve the robust empirical optimization problem
\eqref{eq:robust_empirical_1} and $(x^{\star}(\delta),\,c^{\star}(\delta))$ solve the robust problem \eqref{eq:rob_dgm} under the data generating model $\mathbb P$. Define
\begin{eqnarray*}
A & = & {\mathbb E}_{\mathbb{P}}[- J_\psi (x^{\star}(\delta),\,c^{\star}(\delta))] \in {\mathbb R}^{(d+1)\times (d+1)},\\
B & = & {\mathbb E}_{\mathbb P}[\psi(x^{\star}(\delta),\,c^{\star}(\delta))\,\psi(x^{\star}(\delta),\,c^{\star}(\delta))'] \in {\mathbb R}^{(d+1) \times (d+1)},
\end{eqnarray*}
where $\psi(x,\,c)$ is given by \eqref{eq:psi} and $J_\psi$ denotes the Jacobian matrix of $\psi$, and
\[
V(\delta)
:= A^{-1} B {A^{-1}}' \in {\mathbb R}^{(d+1)\times(d+1)}.
\]
Then $(x_n(\delta),\,c_n(\delta))$ is jointly asymptotically normal where
\begin{eqnarray*}
\sqrt{n}\left[\begin{array}{c} x_n(\delta) -x^{\star}(\delta)\\ c_n(\delta) - c^{\star}(\delta)\end{array}\right] \overset{D}{\longrightarrow} N(0,\,V(\delta)),
\end{eqnarray*}
as $n\rightarrow\infty$.
\end{proposition}

\subsection*{Proof of Theorem \ref{thm:normality2}} The first part of this result concerns the convergence properties of the robust solution, and is a direct consequence of Proposition \ref{prop:consistent} (for consistency of $(x_n(\delta),\,c_n(\delta))$) and Proposition \ref{prop:normality1} (for the asymptotic normality results \eqref{eq:AN}).

The second part of this theorem, which we now prove, characterizes the relationship between the limiting distribution of $(x_n(\delta),\,c_n(\delta))$ and the limiting distribution of the empirical distribution $x_n(0)$ when $\delta$ is small. We begin by showing that $(x^\star(\delta),\,c^\star(\delta))$ exists and is continuously differentiable, following which, we use the first-order conditions \eqref{eq:foc_delta} to obtained the expansion \eqref{eq:rob_asymp_bias} of $(x^\star(\delta),\,c^\star(\delta))$ in the neighborhood of $\delta=0$ and the expression \eqref{eq:pi} for the asymptotic bias $\pi$.

\noindent {\it Existence and continuous differentiability of $(x^\star(\delta),\,c^\star(\delta))$:} We use the Implicit Function Theorem to show existence and determine the smoothness of $(x^{\star}(\delta),\,c^{\star}(\delta))$, and the  first-order conditions to derive the expansion. With a mild abuse of notation, let
\begin{eqnarray*}
g(\delta,\,x, \, c):=\left[\begin{array}{c} g_1(\delta,\,x,\,c) \\ g_2(\delta,\,x,\,c)\end{array}\right] := {\mathbb E}_{\mathbb P}[\psi(x,\,c)].
\end{eqnarray*}
The first-order conditions for \eqref{eq:robust_empirical_1} are
\begin{eqnarray}
g(\delta,\,x,\,c)=0,
\label{foc1}
\end{eqnarray}

Under the assumptions of $\phi(z)$ (see Theorem 3.2 in \cite{GKL}) the convex conjugate $\phi^*(\zeta)$ is twice continuously differentiable in the neighborhood of $\zeta=0$ and satisfies
\begin{eqnarray}
\phi^*(\zeta) = \zeta + \frac{\alpha_2}{2!}\zeta^2 +o(\zeta^2),
\label{eq:phi*}
\end{eqnarray}
where
\[
\alpha_2 = \frac{1}{\phi''(1)}.
\]
It follows that
\[
[\phi^*]'(\zeta) = 1 + \alpha_2 \zeta + o(\zeta)
\]
is continuously differentiable in a neighborhood of $\zeta=0$. This implies that for every fixed $(x,\,c)$, $g(\delta,\,x,\,c)$ is continuously differentiable in $\delta$ in some neighborhood of $\delta=0$, that
\begin{eqnarray*}
g(\delta,\,x, \, c)=\left[\begin{array}{c}{\mathbb E}_{{\mathbb P}}\Big[\nabla_x f(x,Y)\Big] -\frac{\delta}{\phi''(1)}{\mathbb E}_{\mathbb P}\Big[ \Big(f(x,Y)+c\Big)\nabla_x f(x,Y)\Big] + o(\delta) \\[10pt] \mathbb{E}_{{\mathbb P}}\Big[f(x,Y)+c \Big]+ O(\delta) \end{array}\right],
\end{eqnarray*}
and that
\begin{eqnarray}
g\big(0,\,x^{\star}(0),\,-{\mathbb E}_{{\mathbb P}} [f(x^{\star}(0)),\,Y)]\big)=0,
\label{foc1_delta0}
\end{eqnarray}
where $x^{\star}(0)$ is the solution of the empirical problem.
Since $f(x,\,Y)$ is twice continuously differentiable in $x$, $g(\delta,\,x,\,c)$ is continuously differentiable in a neighborhood of $\big(0,\,x^{\star}(0),\,-{\mathbb E}_{{\mathbb P}} [f(x^{\star}(0),\,Y)]\big)$, and
\begin{eqnarray*}
\lefteqn{J_{g,\,(x,\,c)}(\delta,\,x,\,c)\Big|_{(0,\,x^{\star}(0),\,-{\mathbb E}_{{\mathbb P}} [f(x^{\star}(0),\,Y)])}}\\
 & \equiv & \left.\left[\begin{array}{cc}\nabla_x g_1(\delta,\,x,\,c) & \nabla_c g_1(\delta,\,x,\,c) \\
\nabla_x g_2(\delta,\,x,\,c) & \nabla_c g_2(\delta,\,x,\,c)
\end{array}\right]\right|_{(0,\,x^{\star}(0),\, -{\mathbb E}_{{\mathbb P}} [f(x^{\star}(0),\,Y)]))}\\ & = & \left[\begin{array}{cc}{\mathbb E}_{\mathbb P}[\nabla_x^2f(x^{\star}(0),\,Y)] & 0 \\ 0 & 1\end{array}\right]
\end{eqnarray*}
is invertible. It follows from the Implicit Function Theorem that $(x^{\star}(\delta),\,c^{\star}(\delta))$ exists and is continuously differentiable in an open neighborhood of $\big(0,\,x^{\star}(0),\,-{\mathbb E}_{{\mathbb P}} [f(x^{\star}(0),\,Y)]\big)$, so
\begin{align}
x^{\star}(\delta) & = x^{\star}(0) + \pi \delta + o(\delta), \label{eq:temp_exp}\\
c^{\star}(\delta) & = -{\mathbb E}_{{\mathbb P}} [f(x^{\star}(0),\,Y)] + c_1 \delta + o(\delta) \nonumber \\
                  & = -{\mathbb E}_{{\mathbb P}} [f(x^{\star}(0),\,Y)] + O(\delta) \nonumber
\end{align}
in this neighborhood\footnote{It turns out that the value of $c_1$ is not required in our analysis, though if $\phi(z)$ is three times continuously differentiable, it can be shown that $c_1 = -\frac{1}{2}\frac{\phi%^{(3)}
'''(1)}{\left[\phi''(1)\right]^2} \mathbb{V}_{{\mathbb P}}\big[f(x^{\star}(0),\,Y)\big]$.}.

\noindent {\it Taylor series expansion:} We already have \eqref{eq:temp_exp}. To compute $\pi$, observe from \eqref{foc1} that
\begin{eqnarray*}
\lefteqn{{\mathbb E}_{{\mathbb P}} \left[(\phi^*)'\Big(-\delta\left(f(x^{\star}(\delta),Y)+c^{\star}(\delta)\right)\Big)\nabla_x f(x^{\star}(\delta),Y)\right]}\nonumber \\ [10pt]
& = & {\mathbb E}_{{\mathbb P}} \Big[ \nabla_x f(x^{\star}(0),\,Y) + \delta \, \nabla^2_x f (x^{\star}(0),\,Y) \pi + o(\delta)\Big] \\  [10pt]
& & - \frac{\delta}{\phi''(1)}{\mathbb E}_{{\mathbb P}}\Big[\nabla_x f (x^{\star}(0),\,Y) \Big(f(x^{\star}(0),\,Y)-{\mathbb E}_{{\mathbb P}} [f(x^{\star}(0),\,Y)]\Big) + o(\delta) \Big]\\  [10pt]
& =&  \delta\Big\{{\mathbb E}_{\mathbb P}\Big[\nabla^2_x f (x^{\star}(0),\,Y)\Big] \pi
- \frac{1}{\phi''(1)} \mathrm{Cov}_{\mathbb{P}}\big[ f(x^{\star}(0),\,Y),\,\nabla_x f(x^{\star}(0),\,Y)\big] \Big\} +o(\delta)\\ [10pt]
& = & 0.
\end{eqnarray*}
The expression \eqref{eq:pi} for $\pi$ is such that the order $\delta$ term vanishes.

To obtain the expression for $V(\delta)$ observe from \eqref{eq:phi*} that we can write \eqref{eq:psi} as
\begin{eqnarray*}
\psi(x,\,c) = \left[\begin{array}{c} \begin{displaystyle}\nabla_x f(x,\,Y) - \frac{\delta}{\phi''(1)}\Big(f(x,\,Y) + c\Big) \nabla_x f(x,\,Y) + o(\delta)\end{displaystyle} \\  \vspace{-0.4cm}\\f(x,\,Y) + c + O(\delta) \end{array}\right] \equiv \left[\begin{array}{c}\psi_1(x,\,c)\\ \vspace{-0.4cm}\\ \psi_2(x,\,c)\end{array}\right].
\end{eqnarray*}
This implies that $J_\psi$, the Jacobian matrix of $\psi(x,c)$, has
\[
\left[
\begin{array}{ccc|c}
\nabla_{x_1} \psi_1 & \ldots & \nabla_{x_m} \psi_1 & \nabla_c \psi_1 \\
\hline
\nabla_{x_1} \psi_2 & \ldots & \nabla_{x_m} \psi_2 & \nabla_c \psi_2 \\
\end{array}
\right](x^{\star}(\delta),\,c^{\star}(\delta))
=
\left[
\begin{array}{c|c}
\nabla^2_x f(x^{\star}(0),\,Y) + O(\delta) & O(\delta) \\
\hline
\nabla_x f(x^{\star}(0),\,Y)' + O(\delta) & 1 \\
\end{array}
\right]
\]
and hence
\begin{eqnarray*}
A = {\mathbb E}_{\mathbb P}\left[-J_\psi(x^{\star}(\delta),\,c^{\star}(\delta))\right] =
 \left[\begin{array}{cc} - {\mathbb E}_{\mathbb P}\big[\nabla^2_x f(x^{\star}(0),\,Y)\big]
 & 0 \\
0 & -1\end{array}\right] + O(\delta),
\end{eqnarray*}
so
\begin{eqnarray*}
A^{-1} = - \left[\begin{array}{cc}  \begin{displaystyle} { {\mathbb E}_{\mathbb P}\big[\nabla^2_x f(x^{\star}(0),\,Y)\big]}^{-1} \end{displaystyle} & 0 \\
0 & 1\end{array}\right] + O(\delta) = {A^{-1}}'.
\end{eqnarray*}
Likewise
\begin{eqnarray*}
\psi(x^{\star}(\delta),\,c^{\star}(\delta)) = \left[\begin{array}{c} \nabla_x f (x^{\star}(0),\,Y) \\
f(x^{\star}(0),\,Y)-{\mathbb E}_{\mathbb P}[f(x^{\star}(0),\,Y)] \end{array}\right] + O(\delta),
\end{eqnarray*}
so
\begin{eqnarray*}
B & = & {\mathbb E}_{\mathbb P}\big[ \psi(x^{\star}(\delta),\,c^{\star}(\delta))\,
\psi(x^{\star}(\delta),\,c^{\star}(\delta))'\big] \\
& = &
\left[\begin{array}{cc}
{\mathbb V}_{\mathbb P}\big[\nabla_x f(x^{\star}(0),\,Y)\big] & \mathrm{Cov}_{\mathbb{P}}\big[\nabla_x f(x^{\star}(0),\,Y),\,f(x^{\star}(0),\,Y)\big] \\
\mathrm{Cov}_{\mathbb{P}}\big[\nabla_x f(x^{\star}(0),\,Y),\,f(x^{\star}(0),\,Y)\big]' & {\mathbb V}_{\mathbb P}\big[f(x^{\star}(0),\,Y)\big]\end{array}\right] + O(\delta).
\end{eqnarray*}
The expression for $V(\delta)= A^{-1} B (A^{-1})'$ now follows.

\section{Proof of Proposition \ref{prop:expansion}}
\label{sec:TS}

Taylor series implies
\begin{eqnarray*}
f(x+\delta \Delta,\,Y_{n+1})=f(x,\,Y_{n+1}) + \delta \Delta'\,\nabla_x f(x,\,Y_{n+1}) + \frac{1}{2}\delta^2\,\mathrm{tr}\Big(\Delta\Delta'\nabla^2_x f(x,\,Y_{n+1})\Big)+o(\delta^2).
\end{eqnarray*}
We obtain \eqref{eq:mean_f} by taking expectations and noting that $\Delta$ and $Y_{n+1}$ are independent. To derive \eqref{eq:var_f} observe firstly that
\begin{eqnarray*}
\lefteqn{\mathbb{E}_{\mathbb P}\big[f(x+\delta \Delta,\,Y_{n+1})\big]^2 = \mathbb{E}_{\mathbb P}\big[f(x,\,Y_{n+1})\big]^2 + 2\delta \,\mathbb{E}_{\mathbb P}\big[f(x,\,Y_{n+1})\big]\,\mathbb{E}_{\mathbb P}[\Delta]'\,\mathbb{E}_{\mathbb P}\big[\nabla_x f(x,\,Y_{n+1})\big]}  \\[5pt]
& & + \frac{\delta^2}{2}\left\{2\,\mathrm{tr}\Big(\mathbb{E}_{\mathbb P}\big[\Delta\big]\,\mathbb{E}_{\mathbb P}\big[\Delta\big]'\,\mathbb{E}_{\mathbb P}\big[\nabla_x f(x,\,Y_{n+1})\big]\,\mathbb{E}_{\mathbb P}\big[\nabla_x f(x,\,Y_{n+1})\big]'\Big) \right. \\
& & \left.+ 2 \, \mathrm{tr}\Big(\mathbb{E}_{\mathbb P}\big[\Delta\Delta'\big]\,\mathbb{E}_{\mathbb P}\big[f(x,\,Y_{n+1})\big]\,\mathbb{E}_{\mathbb P}\big[\nabla^2_x f(x,\,Y_{n+1})\big]\Big) \right\}  \\[5pt]
& = & \mathbb{E}_{\mathbb P}\big[f(x,\,Y_{n+1})\big]^2 + 2 \delta \mathbb{E}_{\mathbb P}\big[f(x,\,Y_{n+1})\big]\,\mathbb{E}_{\mathbb P}[\Delta]'\,\mathbb{E}_{\mathbb P}\big[\nabla_x f(x,\,Y_{n+1})\big]  \\[5pt]
& & + \frac{\delta^2}{2}\,\mathrm{tr}\left(2 \,\mathbb{E}_{\mathbb P}\big[\Delta\Delta'\big]\left\{\mathbb{E}_{\mathbb P}\big[\nabla_x f(x,\,Y_{n+1})\big] \, \mathbb{E}_{\mathbb P}\big[\nabla_x f(x,\,Y_{n+1})\big]' + \mathbb{E}_{\mathbb P}\big[f(x,\,Y_{n+1})\big] \, \mathbb{E}_{\mathbb P}\big[\nabla^2_x f(x,\,Y_{n+1})\big] \right\} \right. \\  [5pt]
& & \left. - 2 \, {\mathbb V}_{\mathbb P}[\Delta]\,\mathbb{E}_{\mathbb P}\big[\nabla_x f(x,\,Y_{n+1})\big]\,\mathbb{E}_{\mathbb P}\big[ \nabla_x f(x,\,Y_{n+1})\big]' \right),
\end{eqnarray*}
while expectations on both sides of a Taylor series expansion of $[f(x+\delta \Delta,\,Y_{n+1})]^2$ gives
\begin{eqnarray*}
\lefteqn{\mathbb{E}_{\mathbb P}\Big[[f(x+\delta \Delta,\,Y_{n+1})]^2\Big] = \mathbb{E}_{\mathbb P}\Big[[f(x,\,Y_{n+1})]^2\Big] + 2 \delta \,\mathbb{E}_{\mathbb P}[\Delta]' \, \mathbb{E}_{\mathbb P}\Big[f(x,\,Y_{n+1})\, \nabla_x f(x,\,Y_{n+1})\Big]}  \\[5pt]
& & + \frac{\delta^2}{2} \mathrm{tr} \Big(\mathbb{E}_{\mathbb P}[\Delta\Delta']\Big\{2\,\mathbb{E}_{\mathbb P}\big[\nabla_x f(x,\,Y_{n+1})\nabla_x f(x,\,Y_{n+1})'\big] + 2\,\mathbb{E}_{\mathbb P}\big[f(x,\,Y_{n+1})\, \nabla^2_x f(x,\,Y_{n+1})\big]\Big\}\Big) + o(\delta^2).
\end{eqnarray*}
It now follows that
\begin{eqnarray*}
\lefteqn{{\mathbb V}_{\mathbb P}[f(x+\delta \Delta,\,Y_{n+1})] }  \\[5pt]
& = &  {\mathbb V}_{\mathbb P}[f(x,\,Y_{n+1})] + 2 \delta \, \mathbb{E}_{\mathbb P}[\Delta]'\,\mathrm{Cov}_{\mathbb{P}}\big[f(x,\,Y_{n+1}),\, \nabla_x f(x,\,Y_{n+1})\big]  \\[5pt]
& & + \frac{\delta^2}{2}\left\{ \mathrm{tr}\Big( \mathbb{E}_{\mathbb P}[\Delta\Delta']\Big\{ 2 {\mathbb V}_{\mathbb P}\big[\nabla_x f(x,\,Y_{n+1})\big] + 2 \, \mathrm{Cov}_{\mathbb{P}} \big[f(x,\,Y_{n+1}),\,\nabla^2_x f(x,\,Y_{n+1})\big]\Big\}\Big) \right. \\[5pt]
& & \left. + 2 \, \mathrm{tr}\Big({\mathbb V}_{\mathbb P}[\Delta]\,\mathbb{E}_{\mathbb P}\big[\nabla_x f(x,\,Y_{n+1})\big]\,\mathbb{E}_{\mathbb P}\big[\nabla_x f(x,\,Y_{n+1})\big]'\Big)\right\} + o(\delta^2).
\end{eqnarray*}
When $\Delta$ is constant, the definition of a derivative implies \eqref{eq:var_derivatives} and the expansion of ${\mathbb V}_{\mathbb P}[f(x+\delta\Delta,\,Y_{n+1})]$ can be written as \eqref{eq:var_f}.

\section{Proof of Proposition \ref{prop:expansion_nominal}}
\label{sec:expansion_nominal}

We know from Proposition \ref{prop:normality_empirical} that
\eqref{eq:AN_for_emp_sol} holds. It now follows from Proposition \ref{prop:expansion} (with $\Delta\equiv \sqrt{\xi(0)} Z$ and $\delta=\frac{1}{\sqrt{n}}$) that
\begin{eqnarray*}
\lefteqn{
{\mathbb V}_{\mathbb P}\big[f(x_n(0),\,Y_{n+1})\big]
 = {\mathbb V}_{\mathbb P}\big[f\Big(x^{\star}(0)+\sqrt{\frac{\xi(0)}{n}}\big(Z + o_P(1)\big),\,Y_{n+1}\Big)\big]}\\  [5pt]
& = & {\mathbb V}_{\mathbb P}\big[f(x^{\star}(0),\,Y_{n+1})\big]  \\[5pt]
&   & + \frac{2}{\sqrt{n}} {\mathbb E}_{\mathbb P} \Big[\big({\sqrt{\xi(0)}}Z\big)'\Big]\,\mathrm{Cov}_{\mathbb{P}}\big[f(x^{\star}(0),\,Y_{n+1}),\,\nabla_x f(x^{\star}(0),\,Y_{n+1})\big] \\
&   & + \frac{1}{2n} \mathrm{tr}\Big( \mathbb{E}_{\mathbb P}\big[ \sqrt{\xi(0)}Z Z' \sqrt{\xi(0)}'\big]\,\nabla_x^2 {\mathbb V}_{\mathbb P}\big[f(x^{\star}(0),\,Y_{n+1})\big] \\
&   &  + 2 \mathbb{V}\big[\sqrt{\xi(0)} Z\big]\,{\mathbb E}_{\mathbb P}\big[\nabla_x f(x^{\star}(0),\,Y_{n+1})\big]\,{\mathbb E}_{\mathbb P}\big[\nabla_x f(x^{\star}(0),\,Y_{n+1})\big]'\Big) + o\Big(\frac{1}{n}\Big).
\end{eqnarray*}
Noting that $Z$ is standard normal, $\mathbb{E}_{\mathbb P}\big[\sqrt{\xi(0)}Z Z'\sqrt{\xi(0)}'\big]=\xi(0)$, and ${\mathbb E}_{\mathbb P}\big[\nabla_x f(x^{\star}(0),\,Y_{n+1})\big]=0$ (by the definition of $x^{\star}(0)$), we obtain \eqref{eq:asymp_var_nominal}. The expression \eqref{eq:asymp_mean_nominal} for the expected out-of-sample profit under the empirical optimal can be derived in the same way.

\section{Worst case sensitivity}
\label{sec:sensitivity}

Consider the family of probability measures $\{{\mathbb Q}(\delta)\,:\,\delta\geq 0\}$ defined by \eqref{eq:worst-case-Q}.
We begin by characterizing the properties of ${\mathbb Q}(\delta)$, following which we show that the worst-case sensitivity \eqref{eq:sensitivity} is given by the variance of the reward under ${\mathbb P}_n$. This interpretation of the variance and the expansion \eqref{eq:robust-mv-general} allows us to interpret worst-case optimization as a trade-off between maximizing the expected reward under the nominal distribution, and minimizing the worst-case sensitivity to model misspecification. To ease notation, we denote $f=[f_1,\cdots,\,f_n] \equiv[f(x,\,Y_1),\cdots,\,f(x,\,Y_n)]$.
\begin{proposition} \label{prop:sensitivity}
Suppose $\phi(z)$ satisfies Assumption \ref{ass2}.
Then for $\delta>0$ sufficiently small, the worst-case distribution
\begin{eqnarray*}
{\mathbb Q}(\delta)=[q_1^*(\delta),\cdots,\,q_n^*(\delta)],
\end{eqnarray*}
where
\begin{eqnarray}
q_i^*(\delta) = {q_i(c^{\star}(\delta),\,\delta)}
\label{eq:qi*}
\end{eqnarray}
with
\begin{eqnarray}
q_i(c,\,\delta) & = & {p}^n_i\,[\phi']^{-1}\Big(\delta\big[-f_i-c\big]\Big), \nonumber\\
c^{\star}(\delta) & = & \argmax_c \Big\{-\frac{1}{\delta}\sum_{i=1}^n {p}^n_i \phi^*\Big(\delta\big[-f_i-c\big]\Big)-c\Big\}.
\label{eq:c*}
\end{eqnarray}
${\mathbb Q}(\delta)$ is continuously differentiable in a neighborhood of $\delta=0$ with
\begin{eqnarray}
q^*_i(\delta) = {p}^n_i\Big\{1 - \frac{\delta}{\phi''(1)}\Big(f_i-{\mathbb E}_{{\mathbb P}_n}[f]\Big)\Big\} + o(\delta).
\label{eq:Q-expansion}
\end{eqnarray}
\end{proposition}

%\proof{Proof of Proposition \ref{prop:sensitivity}}
\begin{proof}
Let
\begin{eqnarray*}
f^{\star}(\delta)
:= \inf_{\mathbb Q}\Big\{\sum_{i=1}^nq_i f_i + \frac{1}{\delta}\sum_{i=1}^n {p}^n_i \phi\Big(\frac{q_i}{{p}^n_i}\Big)\Big\}.
\end{eqnarray*}
Denote the Lagrangian
\begin{eqnarray*}
L(q,\,c;\,\delta)=\sum_{i=1}^nq_i f_i + \frac{1}{\delta}\sum_{i=1}^n {p}^n_i \phi\Big(\frac{q_i}{{p}^n_i}\Big) + c\Big(\sum_{i=1}^nq_i -1\Big).
\end{eqnarray*}
Then
\begin{eqnarray*}
f^{\star}(\delta) & = & \max_c \min_q L(q,\,c;\,\delta) \\
& = & \max_c \min_q\Big\{\sum_{i=1}^nq_i f_i + \frac{1}{\delta}\sum_{i=1}^n {p}^n_i \phi\Big(\frac{q_i}{{p}^n_i}\Big) + c\Big(\sum_{i=1}^nq_i -1\Big)\Big\},
\end{eqnarray*}
where the minimizer in the second equality
\begin{eqnarray*}
q(c;\,\delta) & = & \argmin_q\Big\{\sum_{i=1}^nq_i f_i + \frac{1}{\delta}\sum_{i=1}^n {p}^n_i \phi\Big(\frac{q_i}{{p}^n_i}\Big) + c\sum_{i=1}^nq_i \Big\} \\ &  = & [q_1(c;\,\delta),\cdots,\,q_n(c;\,\delta)],
\end{eqnarray*}
where
\begin{eqnarray*}
q_i(c;\,\delta) = \argmax_{q_i} \Big\{\frac{q_i}{{p}^n_i}\delta\big[-f_i-c\big]-\phi\Big(\frac{q_i}{{p}^n_i}\Big)\Big\}= {p}^n_i\,[\phi']^{-1}\Big(\delta\big[-f_i-c\big]\Big).
\end{eqnarray*}
By the Inverse Function Theorem, $[\phi']^{-1}(\zeta)$ is continuously differentiable in a neighborhood of $\zeta=0$ since $\phi'(1)=0$ and $\phi''(1)>0$. Note too that if $z(\zeta)=\phi^{-1}(\zeta)$, or equivalently, the solution of $\phi'(z(\zeta))=\zeta$, so $\phi''(z(\zeta))z'(\zeta)=1$ and hence
\begin{eqnarray*}
[\phi^{-1}]'(\zeta)=\frac{1}{\phi''(z(\zeta))}
\end{eqnarray*}
and
\begin{eqnarray}
\phi^{-1}(\zeta)=1 + \frac{\zeta}{\phi''(0)} + o(\zeta).
\label{eq:phiinv-ts}
\end{eqnarray}

Consider now the outer problem in the definition of $f^{\star}$. Observing that
\begin{eqnarray*}
\min_q\Big\{\sum_{i=1}^nq_i f_i + \frac{1}{\delta}\sum_{i=1}^n {p}^n_i \phi\Big(\frac{q_i}{{p}^n_i}\Big) + c\Big(\sum_{i=1}^nq_i -1\Big)\Big\}
= -\frac{1}{\delta}\sum_{i=1}^n {p}^n_i \phi^*\Big(\delta\big[-f_i-c\big]\Big)-c,
\end{eqnarray*}
where
\begin{eqnarray}
\phi^*(\zeta) = \max_z \Big\{\zeta z - \phi(z)\Big\}
\label{eq:convconj}
\end{eqnarray}
is the convex conjugate of $\phi$,
the solution of the outer problem is
\begin{eqnarray}
c^{\star}(\delta) =  \argmax_c \Big\{-\frac{1}{\delta}\sum_{i=1}^n {p}^n_i \phi^*\Big(\delta\big[-f_i-c\big]\Big)-c\Big\}.
\end{eqnarray}
Since the optimizer  $z(\zeta) = [\phi']^{-1}(\zeta)$ in the definition of \eqref{eq:convconj} is continuously differentiable in a neighbourhood of $\zeta=0$, it follows that
\begin{eqnarray*}
\phi^*(\zeta) = z(\zeta)\zeta - \phi(z(\zeta))
\end{eqnarray*}
and $\phi^*(\zeta)$ is differentiable in a neighborhood of $\zeta=0$ with
\begin{eqnarray*}
[\phi^*]'(\zeta) & = & z'(\zeta)\zeta + z(\zeta)-\phi'(z(\zeta))z'(\zeta) \\
& = & z'(\zeta)\big[\zeta-\phi'(z(\zeta))\big]+z(\zeta) \\
& = & z(\zeta),
\end{eqnarray*}
where the second equality follows for the first-order conditions that define $z(\zeta)$.
This implies that $[\phi^*]'(\zeta)$ is also continuously differentiable in a neighborhood of $\zeta=0$ and
\begin{eqnarray*}
[\phi^*]''(\zeta) = z'(\zeta) = \frac{1}{\phi''(z(\zeta))}.
\end{eqnarray*}
Therefore, $\phi^*(\zeta)$ is twice continuously differentiable in a neighborhood of $\zeta=0$ with
\begin{align*}
\left[\phi^*\right]'(0)  & = \left[\phi'\right]^{-1}(0)=1, \\
\left[\phi^*\right]''(0) & = \frac{1}{\phi''(0)},
\end{align*}
so
\begin{eqnarray*}
\phi^*(\zeta)=\zeta + \frac{1}{\phi''(0)}\zeta^2+o(\zeta^2).
\end{eqnarray*}
Returning to $c^{\star}(\delta)$, observe that the first-order conditions corresponding to \eqref{eq:c*} are
\begin{eqnarray}
\sum_{i=1}^n{p}^n_i[\phi^*]'\Big(\delta\big[-f_i-c\big]\Big)-1 = 0,
\label{foc_temp}
\end{eqnarray}
when $\delta>0$. Alternatively, we can define
\begin{eqnarray*}
g(\delta,\,c)
:= \left\{
\begin{array}{ll}
\begin{displaystyle}-\frac{\phi''(1)}{\delta}\Big\{\sum_{i=1}^n{p}^n_i[\phi^*]'\Big(\delta\big[-f_i-c\big]\Big)-1\Big\}, \end{displaystyle} & \delta>0,  \\[5pt]
{\mathbb E}_{{\mathbb P}_n}[f]+c, & \delta=0.
\end{array}\right.
\end{eqnarray*}
Solving \eqref{foc_temp} when $\delta>0$ is equivalent to finding $c^{\star}(\delta)$ such that
\begin{eqnarray}
g(\delta,\,c^{\star}(\delta)) = 0.
\label{eq:fixedpoint}
\end{eqnarray}
One convenient property of $g(\delta,\,c)$ is that it is continuous in a neighborhood of $\delta=0$ for every $c$. In particular, for every fixed $c$
\begin{eqnarray}
g(\delta,\,c) = \sum_{i=1}^n {p}^n_i\big(f_i +c\big) + O(\delta).
\label{eq:fixedpoint-expansion}
\end{eqnarray}
We also have that $(\delta,\,c)=(0,\,-{\mathbb E}_n[f])$ is the solution of
$g(\delta,\,c)=0$. Observing that $g(c,\,\delta)$ is continuously differentiable in $(\delta,\,c)$ in a neighbourhood of $(\delta,\,c)=(0,\,-{\mathbb E}_n[f])$, and
\begin{eqnarray*}
\nabla_c g(0,\,-{\mathbb E}_{{\mathbb P}_n}[f]) = 1,
\end{eqnarray*}
it follows from the Implicit Function Theorem that the solution $c^{\star}(\delta)$ of the equation \eqref{eq:fixedpoint}, and hence \eqref{foc_temp} is continuously differentiable in a neighborhood of $\delta=0$ with $c^{\star}(0)=-{\mathbb E}_{{\mathbb P}_n}[f]$.

Continuous differentiability of both $c^{\star}(\delta)$ and  $[\phi']^{-1}(\zeta)$ implies that $q_i^*(\delta)$ defined by \eqref{eq:qi*} is continuously differentiable in a neighborhood of $\delta=0$. Since $c^{\star}(\delta)=-{\mathbb E}_{{\mathbb P}_{n}}[f] + O(\delta)$, it follows from the expansion \eqref{eq:phiinv-ts} that \eqref{eq:Q-expansion} holds. 
\end{proof}
%\Halmos
%\endproof

\subsection*{Proof of Theorem \ref{thm:sensitivity}}
It follows from \eqref{eq:Q-expansion} that
\begin{align*}
{\mathbb E}_{{\mathbb Q}(\delta)}[f] & = \sum_{i=1}^nq_i(\delta)f_i \\
& = {\mathbb E}_{{\mathbb P}_n}[f]-\frac{\delta}{\phi''(1)}\sum_{i=1}^n{p}^n_i\big(f_i-{\mathbb E}_{{\mathbb P}_n}[f]\big)f_i +o(\delta)\\
& = {\mathbb E}_{{\mathbb P}_n}[f]-\frac{\delta}{\phi''(1)}{\mathbb V}_{{\mathbb P}_n}[f] + o(\delta),
\end{align*}
from which our result follows.

\section{Relationship between the constrained and penalty formulations of the DRO model}
\label{sec:significance}

\subsection*{Relationship between $\epsilon$ and $\delta$} Suppose that ${\mathbb P}_n$ is a given nominal model. Convex duality implies that the solution of the constrained in-sample problem \eqref{eq:constrained} is the solution of the optimization problem  (see \eqref{eq:const-x}--\eqref{eq:const-delta})
\begin{eqnarray*}
(x_n(\epsilon),\,\delta_n(\epsilon))
&  = & \argmax_{\delta\geq 0, \, x} \Big\{ \min_{\mathbb Q}{\mathbb E}_{\mathbb Q}[f(x,\,Y)] + \frac{1}{\delta}{\mathcal H}_\phi({\mathbb Q}\,|\,{\mathbb P}_n)- \frac{\epsilon}{\delta}\Big\} \\
& = & \argmax_{\delta\geq 0,\, x} \Big\{\mathbb{E}_{{\mathbb P}_n}\big[f(x,Y)\big] - \frac{\delta}{2\phi''(1)}\mathbb{V}_{{\mathbb P}_n}\big[f(x,\,Y)\big]  - \frac{\epsilon}{\delta} +  o(\delta) \Big\}
\end{eqnarray*}
where the mean-variance approximation is given in \eqref{eq:robust-mv-general} and applies when $\delta$ is small. Writing $x_n(\epsilon)$ and $\delta_n(\epsilon)$ in orders of $\sqrt{\epsilon}$ it can be shown that
\begin{eqnarray*}
x_n(\epsilon) & = & x_n(0) + \sqrt{\frac{2\phi{''}(1)\epsilon}{{\mathbb V}_{{\mathbb P}_n}[f(x_n(0),\,Y)]}}\,\pi_n+ o(\sqrt{\epsilon}),\\
\delta_n(\epsilon) &= & \sqrt{\frac{2\phi{''}(1)\epsilon}{{\mathbb V}_{{\mathbb P}_n}[f(x_n(0),\,Y)]}}+o(\sqrt{\epsilon}),
\end{eqnarray*}
where $x_n(0)$ is the SAA solution under the nominal distribution ${\mathbb P}_n$ for the in-sample problem, $Y\sim {\mathbb P}_n\equiv [p_1^n,\cdots,\,p_n^n]$, and
\begin{eqnarray*}
\pi_n =  \frac{1}{\phi{''}(1)}\left({\mathbb E}_{{\mathbb P}_n}[\nabla^2_x f(x_n(0),\,Y)]\right)^{-1}\mathrm{Cov}_{\mathbb{P}_n}\big[f(x,\,Y),\,\nabla_x f(x_n(0),\,Y)\big].
\end{eqnarray*}
This implies that the threshold $\epsilon$ corresponding to the robustness parameter $\delta$ is approximately
\begin{eqnarray*}
\epsilon = \frac{{\mathbb V}_{{\mathbb P}_n}[f(x_n(0),\,Y)]}{2\phi{''}(1)}\delta^2.
\end{eqnarray*}

\subsection*{Probabilistic interpretation}

The distribution of ${\mathcal H}_\phi({\mathbb Q}\,|\,{\mathbb P}_n)$ can  be approximated by bootstrapping data $\tilde{Y}_1,\cdots,\,\tilde{Y}_n$ from the nominal distribution ${\mathbb P}_n$ and computing the empirical distribution ${\mathbb Q}=[q_1,\cdots,\,q_n]$ for each such data set. This allows us to interpret $\epsilon$ in \eqref{eq:constrained} as the $1-\alpha$ quantile of the distribution of ${\mathcal H}_\phi({\mathbb Q}\,|\,{\mathbb P}_n)$ where $\alpha$ is such that
\begin{eqnarray}
{\mathbb P}_n\Big[{\mathcal H}_\phi({\mathbb Q}\,|\,{\mathbb P}_n) \leq \epsilon\Big] \approx 1-\alpha.
\label{eq:prob}
\end{eqnarray}

As seen from our examples in Section \ref{sec:applications}, it is difficult to estimate the associated quantiles $1-\alpha$ from simulation when $\delta$ or $\epsilon$ are in the desirable part of the robust mean-variance frontier. If one still wishes to estimate $\alpha$ for a particular $\delta$ or $\epsilon$, we can use the asymptotic properties of $\phi$-divergence if $\phi(z)$ is sufficiently regular and the support of the data generating mechanism $\mathbb P$ is finite (see \cite{ben2013}). Under these conditions
\begin{eqnarray*}
\frac{2n}{\phi''(1)} {\mathcal H}_\phi({\mathbb Q}\,|\,{\mathbb P}_n) ~\sim \chi^2_{n-1}
\end{eqnarray*}
(the $\chi^2$ distribution with $n-1$ degrees of freedom) so \eqref{eq:prob} is satisfied if $\epsilon$ and $\alpha$ are such that $2n\epsilon = \chi^2_{n-1,\,1-\alpha}$, the $1-\alpha$ quantile of $\chi^2_{n-1}$. If our goal is to ``optimize" the trade-off between the out-of-sample mean and sensitivity of the reward, however, there seems to be little reason to formulate and/or interpret uncertainty sets through confidence levels $\alpha$.

\end{document}